\documentclass[times,twocolumn,final,authoryear]{elsarticle}

\usepackage{amsfonts,amsmath,amsthm,amssymb,bm}
\usepackage{latexsym}

\usepackage{color, xcolor, colortbl}
\definecolor{newcolor}{rgb}{.8,.349,.1}

\journal{Pattern Recognition Letters}

\usepackage{bibunits}

\usepackage{enumitem}
\usepackage{subcaption}

\usepackage{geometry}%
\usepackage{hyperref,url}
\hypersetup{
  colorlinks,
  citecolor=blue,
  filecolor=blue,
  linkcolor=blue,
  urlcolor=blue
}

\newtheorem{theorem}{Theorem}

\newtheorem{corollary}[theorem]{Corollary}

\newtheorem{lemma}[theorem]{Lemma}

\usepackage{enumitem}
\usepackage[referable,para]{threeparttablex}
\usepackage{footnote}
\usepackage{booktabs}  

\DeclareMathOperator{\MEAN}{E}
\DeclareMathOperator{\E}{E}

\newcommand{\RMS}{{\mathrm{RMS}}}
\newcommand{\MSE}{{\mathrm{MSE}}}

\newcommand{\Err}{Err}

\newcommand\Item[1][]{%
  \ifx\relax#1\relax \item \else \item[#1] \fi
\abovedisplayskip=0pt\abovedisplayshortskip=0pt~\vspace*{-\baselineskip}}

\begin{document}

\newcommand{\APPENDIX}{Appendix}

\begin{bibunit}[model2-names]
  \begin{frontmatter}
  \title{On The Smoothness of Cross-Validation-Based Estimators Of Classifier Performance} \author[WAY]{Waleed~A.~Yousef\corref{cor1}}
  \ead{wyousef@UVIC.ca} \cortext[cor1]{Corresponding Author}

  \address[WAY]{ECE Dep., University of Victoria, Canada; CS Dep., HCILAB, Faculty of Computers and Artificial Intelligence, Helwan
    University, Egypt.}

  \begin{abstract}
    Using the influence function (IF) to estimate the standard error of the bootstrap (BS)-based estimators, that estimate the
    classifiers' performance in terms of the error rate and the Area Under the ROC curve (AUC), was very successful. Only the
    ``smooth'' and differentiable versions of these BS-based estimators were eligible for IF approach. Since then, this successful
    approach is not extended to the cross validation (CV)-based estimators, which are more important because of their ubiquitous use in
    the majority of applications and by the majority of practitioners. The literature includes many versions of CV, and each version
    has different variants. The smoothness of these versions and variants needs to be studied before adopting the IF for any of them,
    which is the motivation behind the present article. First, we start by a mathematical formalization of many of these different
    versions and variants, that estimate the error rate and the AUC of a classification rule, to show the difference among them and the
    connection to the smooth version of the BS-based estimators. Second, we prove some of their properties and conclude that many of
    these variants are redundant and ``not smooth'', and therefore are not eligible for the IF approach, except the repeated $K$-fold
    CV (CVKR) and the Monte-Carlo CV (CVKM). For pedagogy and deeper understanding, we provide a novel experimental example that
    illustrate the smoothness property. Third, we discuss these versions and variants in terms of some known results in the literature,
    including effective training sample size, weak correlation, and mean squared estimation error; then, we conclude that there is
    still no final solid advice for practitioners, on the preference of a particular estimator, until a robust IF method is developed
    and a comprehensive comparative study is conducted.
  \end{abstract}

\end{frontmatter}


  \section{Introduction}\label{sec:introduction}
Using the influence function (IF) to estimate the standard error of the bootstrap (BS)-based estimators, that estimate the classifiers'
error rate was introduced in~\cite{Efron1997ImprovementsOnCross}. Then, it was extended by~\cite{Yousef2005EstimatingThe} to estimate
the standard error of the BS-based estimators that estimate the classifiers' Area Under the ROC curve (AUC). These BS-based estimators
were ``smooth'' and differentiable, as will be explained later, and therefore their standard error was estimable using the IF method
very successfully. However, the majority of practitioners, if not all, prefer the CV-based estimators because of their simplicity. Even
some pioneers of the field prefer CV because they are easier to explain to practitioners\footnote{Personal communication with Robert
  Tibshirani, Washington, DC., circa 2004, who is himself a coauthor of \cite{Efron1997ImprovementsOnCross} that adopts the BS and IF
  approaches.}. Therefore, it is quite important to extend the very successful IF approach to CV-based estimators. However, many
versions of CV exist in the literature and are even used by practitioners loosely without understanding the connection or difference
among them. To name a few: leave-one-out CV, $K$-fold CV, repeated CV, and Monte-Carlo CV, among other. Each of these versions has some
variants (e.g., whether to first sum over folds or over observations, as will be explained later), which may be thought of as trivially
equivalent. However, as will be shown in Section~\ref{sec:analysis} they are not. Before applying the IF method to these estimators we
have to study their properties, in general, and their smoothness property, in particular. A statistic is called smooth if a small
change in the observations leads to a small change in the statistic. Therefore, for instance, the sample mean is a smooth statistic,
whereas the sample median is not. The IF approach leverages the derivative of an estimator with respect to the perturbation of the
probability measure of the sample. The main contribution of the present article is to study the connection and differences among these
CV-based estimators and their connection to the smooth BS-based estimators, on one hand, and to investigate their smoothness and
differentiablity, and hence their suitability for the IF approach, on the other hand.

\medskip

The rest of this article is organized as follows.ection~\ref{sec:analysis} provides the mathematical formalization of the main versions
and variants of the CV-based estimators, namely: leave-one-out, $K$-fold, repeated $K$-fold, and Monte-Carlo, in addition to the
BS-based estimator because of its close connection to the CV and it known smoothness property. Section~\ref{sec:experimental-results}
explains theoretically and experimentally the smoothness and differentiablity property of these estimators.
Section~\ref{sec:theoritical-critique} discusses more properties of these estimators in the light of some known results in the
literature. Section~\ref{sec:conclusion} concludes the article and suggests the future research to be able to provide a final concrete
advice for practitioners. The~\APPENDIX~contains all the lemmas and their proofs, which are deferred to it because of the page size
limit and to make the article as lucid as possible to all readers, specially practitioners.



  \section{Mathematical Formalization of Estimators}\label{sec:analysis}
We treat the binary classification problem, where a classification rule is trained on the training set $\mathbf{X}$ and provides a
score $h_{\mathbf{X}}(x)$ when tested on an observation $x$, which will be classified as either belonging to class 1 (or class 2) if
$h_{\mathbf{X}}(x)$ is smaller (or larger) than a chosen threshold $th$. Two important performance measures, usually used to assess the
classification rules, are the error rate and the area under the ROC curve (AUC). For the error rate, we use the notation $Q\left(
  x,\mathbf{X}\right)$, where $Q$ is the 0-1 loss function to denote miss (or correct) classification of $x$. For the AUC (also
called the MannWhitney statistic, a form of the rank tests), we use the notation%
\begin{align}
  \widehat{AUC} &  =\frac{1}{n_{1}n_{2}}\sum_{i}\sum_{j} \psi\left(  h_{\mathbf{X}}(x_{i})  ,h_{\mathbf{X}}(y_{j})  \right),\label{eq:MWStat}
\end{align}
where $\psi(a,b) = 0,\ 0.5,\ 1$ when $a>b,\ a=b,\ a<b$, respectively. This means that the classifier is trained on the training set
$\mathbf{X}$ and then tested and produced scores on the $n_1$ and $n_2$ observations, from class 1 and class 2, respectively. Then, for
each pair of scores (out of the $n_1 \times n_2$ possible pairs) a single value of the MannWhitney kernel $\psi$ is calculated. In
terms of the $U$-statistic and rank test theories~\citep{Randles1979IntroductionTo, Hajek1999TheoryOfRank}, the error rate is a
one-sample statistic because it requires only a single observation, from either classes, to evaluate, whereas the AUC is a two-sample
statistic because it requires two observations, one from each class, to evaluate.

\medskip

As known for all resampling-based estimators, BS, CV, etc., we need to resample (draw a sub-sample) from the given sample that consists
of observations from both classes. Therefore, we have two possible resampling mechanisms: (1) ``with-stratification'', also called
``separate-sampling'', in which the original sample ($n$ observations) is split into two samples ($n_1$ and $n_2$ observations), each
includes observations only from one class, from each we resample independently. This guarantees obtaining observations from each class.
(2) ``without-stratification'', also called ``random-sampling'', in which we keep the original sample without splitting, then resample
directly from it as a single pool. This does not guarantee obtaining observations from each class. In principle, and aside from the
theoretical or the computational utility, we can use any of these two resampling mechanisms, regardless whether we estimate the error
rate or the AUC. However, in the present article we avoid redundancy and use, without any loss of generality (WLOG),
with-stratification for the AUC (a two-sample statistic) and without-stratification for the error rate (a one-sample statistic).

\subsection{Leave-One-Out Cross Validation (LOOCV)}\label{sec:leave-one-out-1}
\subsubsection{Error rate}
In LOOCV, the observation $x_{i}$ is excluded, then the classifier is trained on the remaining dataset $\mathbf{X}_{\left(i\right)} =
x_{1},\dots,x_{i-1},x_{i+1},\dots,x_{n}$ (or, for short, $\mathbf{X}_{(i)}=\left\{ x_{i'}:i' \neq i\right\}$), then tested on $x_{i}$.
The procedure is repeated $\forall i$ and the average over all observations is calculated. Therefore, LOOCV for the error rate is
formally defined as%
\begin{align}
  \widehat{Err}^{\left(CVN\right)  } & =\frac{1}{n}\sum_{i=1}^{n}Q\left(x_{i},\mathbf{X}_{\left(  i\right)  }\right).\label{eq:ErrCVN}
\end{align}
We denote it CVN, as opposed to some literature where it is called CV1, for consistency of notation, because it is a special case of
the $K$-fold CV (CVK) of the next section. LOOCV can be seen as partitioning the dataset of size $N$ to $N$ folds, each containing just
one observation.

\subsubsection{AUC}
As introduced above, stratification is used for AUC estimation, and the two datasets are resampled independently, one observation is
left out from each class to produce the two sets: $\mathbf{X1}_{(i)}=\left\{ x_{i'}:i' \neq i\right\},\ \mathbf{X2}_{(j)}=\left\{
  y_{j'}:j' \neq j\right\}$. Therefore, it is called a leave-pair-out estimator~\citep{Yousef2004ComparisonOf,
  Yousef2005EstimatingThe}. The classifier is trained on the set $\mathbf{X1}_{(i)} \cup \mathbf{X2}_{(j)}$ (in the sequel, the $\cup$
operator will be dropped for short), then tested on the two left out observations $x_i$ and $y_j$ to produce corresponding two scores,
from which the Mannwhitney statistic~\eqref{eq:MWStat} is calculated. The procedure is repeated $\forall i,j$ and the average over all
possible pairs is calculated. Hence, the CVN for AUC is formally defined as%
\begin{align}
  \widehat{AUC}^{\left(  CVN\right)  }  &  =\frac{1}{n_{1}n_{2}} \sum_{j=1}^{n_{2}} \sum_{i=1}^{n_{1}} \psi\left(h_{\mathbf{X1}_{\left(i\right)}\mathbf{X2}_{\left(j\right)}}(x_i) ,h_{\mathbf{X1}_{\left(i\right)}\mathbf{X2}_{\left(j\right)}}(y_j)\right).\label{eq:AUCCVN}
\end{align}

\subsection{Conventional (one run) K-fold CV (CVK)}\label{sec:conventional-one-run}
\subsubsection{Error rate}
The CVK is carried out by partitioning the dataset into $K$ partitions (folds), training on $K-1$ of them, then testing on the
remaining one. The partitions can be defined as a mapping that maps each observation to a partition: $\mathcal{K}:\left\{
  1,\dots,n\right\} \mapsto\left\{ 1,\dots,K\right\} ,~K=n/n_{K}$, then%
\begin{align}
  \mathcal{K}\left(  i\right)  =k,\quad n_{K}\left(  k-1\right)  <i\leq n_{K}k,\quad k=1,\dots,K, \label{EqKmapDef}
\end{align}
where $\sum_{i}I_{\left( \mathcal{K}\left( i\right) =k\right) }=n_{K}~\forall k$. For simplifying notation, and WLOG, $n_K$ is assumed to
be integer value. Then, the $k\textsuperscript{th}$ excluded partition is the set of observations
$\left\{x_{i}:\mathcal{K}\left(i\right) = k\right\}$ and the remaining dataset is $\mathbf{X}_{\left( \left\{ k\right\} \right)
}=\left\{ x_{i}:\mathcal{K}\left( i\right) \neq k\right\}$. The CVN is, therefore, a special case of the CVK, where
$n_{K}=1,~K=n,~\mathcal{K}\left( i\right) =i,~\mathbf{X}_{\left( \left\{ i\right\} \right) }=\left\{x_{i^{\prime}}:\mathcal{K}\left(
    i^{\prime }\right) \neq i\right\} =\left\{x_{i^{\prime}}:i^{\prime}\neq i\right\}=\mathbf{X}_{\left( i\right) }$. The classifier is
trained on $\mathbf{X}_{\left( \left\{ k\right\} \right) }$ and tested on the $k\textsuperscript{th}$ partition. The procedure is
repeated $\forall k$ and the average is calculated, once, over all tested observations. Therefore, the CVK estimator for the error rate
is formally defined as%
\begin{align}
  \widehat{Err}^{\left(CVK\right)} & =\frac{1}{n}\sum_{i=1}^{n}Q\left(x_{i},\mathbf{X}_{\left(  \{\mathcal{K}\left(  i\right)  \}\right)  }\right).\label{EQCVK}
\end{align}
However, the CVK could have been defined differently, denoted by the partitioned variant CVK*, as%
\begin{align}
  \widehat{Err}^{\left(CVK*\right)} & =\frac{1}{K} \sum_{k} \left[\frac{1}{n_{K}} \sum_{i \in \mathcal{K}^{-1}\left( k\right) } Q\left( x_{i},\mathbf{X}_{\left(\left\{k\right\}\right)}\right) \right].\label{EQCVKpart}
\end{align}
We show that (\APPENDIX) $\widehat{Err}^{\left(CVK\right)} = \widehat{Err}^{\left(CVK*\right)}$. The former is the ``pooled'' variant,
where the average is taken once over all observations. The latter is the ``partitioned'' variant, where the average is taken over the
observations of each partition, followed by another average over all partitions.

\subsubsection{AUC}
For the sake of generality, we assume that the number of observations $n_{1}$ and $n_{2}$, and the number of partitions $K_{1}$ and
$K_{2}$, are not necessarily equal. This allows the CVN to follow naturally from the CVK as a special case by setting $K_{1}=n_{1}$ and
$K_{2}=n_{2}$. Then, we have $\mathcal{K}_{1}:\left\{ 1,\dots ,n_{1}\right\} \mapsto\left\{ 1,\dots,K_{1}\right\},~K_{1}=n_{1}/n_{1K}$,
such that $\mathcal{K}_{1}\left( i\right) =k,~k=1,\dots,K_{1}$, and $\sum_{i=1}^{n_1}I_{\left(\mathcal{K}_{1}\left( i\right) =k\right)
}=n_{1K}~\forall k$. Similarly, the other partitioning function $\mathcal{K}_2$ is defined. More compactly, we could have written the
two partitions in one definition using a subscript $c=1,2$ for the two classes, respectively, as: $\mathcal{K}_{c}:\left\{ 1,\dots
  ,n_{c}\right\} \mapsto\left\{1,\dots,K_{c}\right\},~K_{c}=n_{c}/n_{cK}$, such that $\mathcal{K}_{c}(i)=k_c,~k_c=1,\dots,K_c$, and
$\sum_{i=1}^{n_c}I_{\left(\mathcal{K}_c(i) = k\right) }=n_{cK}~\forall k,\ c=1,2$. The classifier is trained on all partitions except a
partition from each class; i.e.\ the training set is
$\mathbf{X1}_{\left(\left\{k_{1}\right\}\right)}\mathbf{X2}_{\left(\left\{k_{2}\right\}\right)}$. Afterwards, the classifier is tested
on the two left-out partitions to produce corresponding two sets of $n_{1K}$ and $n_{2K}$ scores, from which $n_{1K} \times n_{2K}$
pairs of Mannwhitney statistic~\eqref{eq:MWStat} are calculated. The procedure is repeated $\forall k_1, k_2$, i.e., $K_1 \times K_2$
times, and the average MannWhitney is calculated once over the $n_1 \times n_2$ pooled pairs. Therefore, the CVK estimator for the AUC
is given by%
\begin{align}
  \widehat{AUC}^{\left(  CVK\right)  }  &  =\frac{1}{n_{1}n_{2}} \sum_{j=1}^{n_{2}}\sum_{i=1}^{n_{1}}\psi\left(  h(x_{i})  , h(y_{j})  \right),\label{eq:CVK-AUC}
\end{align}%
where $h = h_{\mathbf{X1}_{\left( \left\{ \mathcal{K}_{1}(i)\right\} \right) }\mathbf{X2}_{\left(\left\{\mathcal{K}_{2}(j)\right\} \right)
  }}$. Exactly as was done for the error rate above, the partitioned variant for estimating the AUC could have been defined as%
\begin{align}
  \widehat{AUC}^{\left(CVK*\right)}& =\frac{1}{K_{1}K_{2}}\sum_{k_{1}=1}^{K_{1}}\sum_{k_{2}=1}^{K_{2}}\times\notag\\
                                  &\left[\frac{1}{n_{1K}n_{2K}}\sum_{i\in\mathcal{K}_{1}^{-1}(k_{1})}\sum_{j\in\mathcal{K}_{2}^{-1}(k_{2})}\psi\left(h(x_{i}), h(y_{j})\right)\right],\label{eq:CVK*-AUC}
\end{align}
where $h = h_{\mathbf{X1}_{\left(\left\{k_{1}\right\}\right)}\mathbf{X2}_{\left(\left\{k_{2}\right\}\right)}}$. We show that (\APPENDIX)
$\widehat{AUC}^{\left( CVK\right) } = \widehat{AUC}^{\left(CVK*\right)}$ .

A third AUC CVK variant exists, which we observed from experience and personal communication with some practitioners. It is tempting to
use this variant because of its reduced computational speed. It is exactly the same as the partitioned variant CVK* but with enforcing
the two left-out partitions to have the same index. This means that the classifier will be trained (and tested) only $K$ times, rather
than $K^2$ times, each time on the two partitions $k_1 = k_2 = k= 1,\dots, K$. The formal definition is
\begin{align}
  \widehat{AUC}_{reduced}^{\left(  CVK *\right)  } &  =\frac{1}{K}\sum_{k=1}^{K}\left[ \frac{1}{n_{1K}n_{2K}} \sum_{i\in\mathcal{K}_{1}^{-1}(k)} \sum_{j\in\mathcal{K}_{2}^{-1}(k)} \psi\left(h(x_{i}), h(y_{j})\right)\right],\notag
\end{align}%
where $h=h_{\mathbf{X1}_{\left( \left\{ k\right\}\right) }\mathbf{X2}_{\left( \left\{ k\right\} \right) }}$. Obviously, it does not equal
to, and it exhibits more variance than, the variant~\eqref{eq:CVK*-AUC}. This rises from the fact that AUC is a two sample
$U$-statistic, whose variance is minimized only if it is averaged over all possible permutations of the two sets of observations, i.e.,
$n_{1} \times n_{2}$ times, which is guaranteed by the variant~\eqref{eq:CVK*-AUC}. However, for the reduced variant, the AUC kernel is
evaluated only $K(n_{1K} \times n_{2K})$ times. Clearly, this concern is not raised when we formalized the CVK estimator for the error
rate since it is a one-sample statistic.

\subsection{Repeated (Randomized) K-fold CV (CVKR)}\label{sec:repe-rand-k}
\subsubsection{Error rate}
This version of CV is $M$-time repetitions of the conventional CVK. In every repetition, there is a new randomized partition function
$\mathcal{K}_m(i) =\mathcal {K}\left(\mathcal{R}_{m}(i)\right)$, where $\mathcal{K}$ is defined in~\eqref{EqKmapDef} and
$\mathcal{R}_{m}$ is the $m^{\text{th}}$ one-to-one randomized mapping of $\left\{ i,\dots,n\right\} $ onto itself. In words, the same
conventional partitioning of the CVK is applied each iteration after randomly shuffling the dataset. The training sets, then, are
defined for each iteration $m$ as $\mathbf{X}_{\left( \left\{ k\right\} ,m\right) }=\left\{ x_{i}:\mathcal{K}_{m}\left( i\right) \neq
  k\right\}$. The conventional (one run) CVK follows naturally from the CVKR by setting $M=1$. In each of the $M$ repetitions, the
classifier is trained and tested using the conventional CVK. The scores are stored for all $n$ observations, and the average is taken
over all $M \times n$ scores. Therefore, formally
\begin{align}
  \widehat{Err}^{\left(  CVKR\right)  }  &  =\frac{1}{n}\sum_{i=1}^{n}\left[\frac{1}{M}  \sum_{m}Q\left(  x_{i},\mathbf{X}_{\left(  \left\{  \mathcal{K}_{m}\left(  i\right)  \right\}  ,m\right)  }\right)\right].\label{EqCVKRa}
\end{align}
On the other hand, the partitioned variant of the CVKR, where average is taken over repetitions, is defined as:%
\begin{align}
  \widehat{Err}^{\left(CVKR*\right)} &  =\frac{1}{M}\sum_{m=1}^{M} \left[\frac{1}{K}\sum_{k}  \frac{1}{n_{K}} \sum_{i\in\mathcal{K}_{m}^{-1}\left(k\right)} Q\left(x_{i},\mathbf{X}_{\left(\left\{k\right\} ,m\right)}\right)\right].\label{EqCVKRpart}
\end{align}
We show that (\APPENDIX) $\widehat{Err}^{\left( CVKR\right) } = \widehat{Err}^{\left(CVKR*\right)}$.

\subsubsection{AUC}
With stratification, we can define the two partitions using compact notation as: $\mathcal{K}_{cm}(i) =\mathcal
{K}\left(\mathcal{R}_{cm}(i)\right),\ c=1,2$ for the two classes respectively, where $\mathcal{K}$ is given by~\eqref{EqKmapDef} and
$\mathcal{R}_{cm}$ is the randomized function defined above for class $c$. The training sets are given by $\mathbf{Xc}_{\left( \left\{
      k\right\} ,m\right) }=\left\{ x_{i}:\mathcal{K}_{cm}\left( i\right) \neq k,~x_{i}\in \mathbf{Xc}\right\}$. Then, the CVKR
estimator is defined for the AUC as:%
\begin{align}
  \widehat{AUC}^{\left(  CVKR\right)  }  &  =\frac{1}{n_{1}n_{2}}\sum_{j=1}^{n_{2}}\sum_{i=1}^{n_{1}}\left[\frac{1}{M} \sum_{m}\psi\left(h\left(x_{i}\right), h\left(  y_{j}\right)  \right)\right],\label{eq:CVKR-AUC}
\end{align}
where $h= h_{\mathbf{X1}_{\left( \left\{ \mathcal{K}_{1m}(i)\right\}\right) }\mathbf{X2}_{\left( \left\{\mathcal{K}_{2m}(j)\right\}
    \right)}}$. The partitioned variant of this estimator is defined as:%
\begin{align}
  \widehat{AUC}^{\left(  CVKR*\right)  }  & =\frac{1}{M}\sum_{m=1}^{M}\left[\frac{1}{K_{1}K_{2}}\sum_{k_{1}=1}^{K_{1}}\sum_{k_{2}=1}^{K_{2}} \frac{1}{n_{1K}n_{2K}}\right.\notag\\
                                         &\left.\sum_{i\in\mathcal{K}_{1m}^{-1}(k_{1})}\sum_{j\in\mathcal{K}_{2m}^{-1}(k_{2})}\psi\left(h( x_{i}), h( y_{j})\right)\right],\label{eq:CVKR-AUCstar}
\end{align}
where $h = h_{\mathbf{X1}_{\left( \left\{ k_{1}\right\} \right) }\mathbf{X_2}_{\left( \left\{k_{2}\right\} \right) }}$. We show that
(\APPENDIX) $\widehat{AUC}^{\left( CVKR\right) } = \widehat{AUC}^{\left( CVKR*\right) }$.

\subsection{Monte-Carlo (MC) K-fold CV (CVKM)}\label{sec:monte-carlo-mc}
\subsubsection{Error rate}
CVKM is similar to CVKR except in that, in every repetition (called in literature MC run, or trial) there is only one partition to test
on. More clearly, there is only one training set of size $n-n_{K}$, and one testing set of size $n_{K}$. To be able to formalize and
analyze this version CV we proceed as follows. Because in each MC trial the observations are randomized, then WLOG we index the single
testing partition with $k=1$; i.e., testing is always carried out on the first partition. Therefore, the training set is
$\mathbf{X}_{\left( \left\{ 1\right\} ,m\right) }$; which, obviously, may or may not include an observation $x_{i}$. Define $I_{i}%
^{m}=I_{(\mathcal{K}_{m}(i)=1)}$, a 0-1 condition for checking whether an observation $x_{i}$ appears in the testing partition or not. We
did not need this condition for CVKR, because every MC trial includes $K$ different training and testing sets, and for every $x_{i}$
there is a corresponding test set $\mathbf{X}_{\left( \left\{ \mathcal{K}_{m}\left(i\right) \right\} ,m\right) }$ that does not include
it. It is clear that $\sum_{i}I_{i}^{m}=n_{K}$, the number of observations in any partition. The CVKM estimator can then be formally
defined as%
\begin{align}
  \widehat{Err}^{\left(  CVKM\right)  }=\frac{1}{n}\sum_{i}\left[  \left.\sum_{m}I_{i}^{m}Q\left(  x_{i},\mathbf{X}_{\left(  \left\{  1\right\},m\right)  }\right)  \right/  \sum_{m'}I_{i}^{m'}\right]  , \label{EqCVKM}
\end{align}
which means that for every observation we train once, then test on that observation only if it is not included in the training set. We
could have defined the MC estimator so that for every MC loop we train once and test on all of the observations in the testing set;
then we average over the MC trials. This partitioned variant should be defined as:%
\begin{align}
  \widehat{Err}^{\left(  CVKM\ast\right)  }  &  =\frac{1}{M}\sum_{m=1}^{M}\left[  \frac{1}{n_{K}}\sum_{i\in\mathcal{K}_{m}^{-1}\left(  1\right)}Q\left(  x_{i},\mathbf{X}_{\left(\{1\}, m\right)}\right)\right].\label{EqCVKM*a}
\end{align}
We show that (\APPENDIX) $\widehat{Err}^{\left( CVKM\right) } \neq \widehat{Err}^{\left( CVKM\ast\right) }$ under finite $M$, which is in
contrast to CVK and CVKR. However, asymptotically as $M\rightarrow \infty$, they are equal. It is worth mentioning that the two
variants are very similar (mathematically) to the two BS variants discussed in Section~\ref{sec:leave-one-out}.

\subsubsection{AUC}
With stratification, for each MC run there is only one partition from each class to test on. Following the same argument above, we
set the two testing partitions as $k_{1}=k_{2}=1$. The training sets are therefore $\mathbf{Xc}_{\left( \{1\},m\right) },\ c=1,2$. We
define, $I_{i}^{m}=I_{(\mathcal{K}_{1m}(i) = 1)}$ and $I_{j}^{m}=I_{(\mathcal{K}_{2m}(j)=1)}$. Then, the CVKM variant for the AUC is
defined as%
\begin{align}
  \widehat{AUC}^{\left(  CVKM\right)  } &=\frac{1}{n_{1}n_{2}}\sum_{j}\sum_{i}\left[  \left.  \sum_{m}I_{j}^{m}I_{i}^{m}\psi\left(h(x_{i}), h(y_{j})\right)\right/  \sum_{m'}I_{i}^{m'}I_{j}^{m'}\right],\label{eq:CVKM-AUC}
\end{align}
where $h = h_{\mathbf{X1}_{\left(\{1\},m\right)}\mathbf{X2}_{\left(\{1\},m\right)}}$. Similar to the treatment above of the error rate, we
could have defined the partitioned variant as%
\begin{align}
  \widehat{AUC}^{\left(  CVKM\ast\right)  }  &  =\frac{1}{M}\sum_{m=1}^{M}\left[  \frac{1}{n_{1K}n_{2K}}\sum_{j\in\mathcal{K}_{2m}^{-1}\left(1\right)  }\sum_{i\in\mathcal{K}_{1m}^{-1}\left(  1\right)  }\psi\left(h( x_{i}), h( y_{j})\right)\right],\label{eq:AUC-CVKM*}
\end{align}
where $h = h_{\mathbf{X1}_{\left(\{1\},m\right)}\mathbf{X2}_{\left(\{1\},m\right)}}$. We show that (\APPENDIX) the two variants are
equal only asymptotically, but not for finite $M$.

\subsection{Leave-One-Out Bootstrap (LOOB)}\label{sec:leave-one-out}
\subsubsection{Error rate}
The BS-based estimators was introduced in~\citep{Efron1983EstimatingTheError, Efron1993AnIntroduction, Efron1997ImprovementsOnCross},
and summarized in~\cite{Yousef2022MachineLearning}. The LOOB estimator for the error rate was defined as%
\begin{align}
  \widehat{Err}^{\left( 1\right) } & =\frac{1}{n}\sum_{i}\left[ \sum _{b}I_{i}^{b}Q\left( x_{i},\mathbf{X}^{\ast b}\right) \left/ \sum_{b'} I_{i}^{b'}\right.  \right],\label{EqLOOerr}
\end{align}%
where $I_{i}^{b}$ equals $1$ if the observation $x_{i}$ did not appear in the $b^{\text{th}}$ BS, and $0$ otherwise. This is to avoid the
bias occurring from training and testing on the same observations. The classifier is trained on the $b\textsuperscript{th}$ BS,
$\mathbf{X}^{\ast b}$, and tested on the observation $x_{i}$; the procedure is repeated for $B$ bootstraps. The partitioned variant was
defined as%
\begin{align}
  \widehat{Err}^{\left( \ast\right) } & =\frac{1}{B}\sum_{b}\left[ \sum _{i}I_{i}^{b}Q\left( x_{i},\mathbf{X}^{\ast b}\right) \left/ \sum_{i'} I_{i'}^{b}\right.  \right].\label{EqStarErr}
\end{align}
Both variants~\eqref{EqLOOerr} and~\eqref{EqStarErr} estimate the expectation (over the training and testing sets) of the error rate,
but with different summation order, to approximate the two empirical distributions of the BS replications and the left-out
observations. Both variants establish the parallelism with the CV variants of previous sections. However, it can be proven that
(\APPENDIX) (1) they are not equal either for finite bootstraps $B$ or asymptotically as $B\rightarrow\infty$. (2) The ratio between
the two estimators does not depend on the classifier; rather it depends on the sampling mechanism (exactly as was the case for the
CVKM).

\subsubsection{AUC}
The two variants of the BS estimator for the AUC, which are analogue to those of~\eqref{EqLOOerr} and~\eqref{EqStarErr} for the
error rate, were introduced and defined in our work~\cite{Yousef2004ComparisonOf, Yousef2005EstimatingThe} respectively. The history of
the theoretical development is summarized in~\cite{Yousef2022MachineLearning}. With stratification, we had defined the leave-pair-out
BS (LPOB) for the AUC as%
\begin{align}
  \widehat{AUC}^{\left(  1,1\right)  }  &  =\frac{1}{n_{1}n_{2}}\sum_{i}\sum_{j}\sum_{b}\frac{I_{i}^{b}I_{j}^{b} \psi\left(h(x_i), h(y_j)\right)}{\sum_{b'}I_{i}^{b'}I_{j}^{b'}},\label{eq:BS-AUC}
\end{align}%
where $h = h_{\mathbf{X}^{\ast b}}$. On the other hand, we had defined the partitioned variant for the AUC as%
\begin{align}
  \widehat{AUC}^{\left( \ast\right) } & =\frac{1}{B}\sum_{b}\left[ \sum _{j}\sum_{i}I_{i}^{b}I_{j}^{b}\psi\left( h( x_{i}) , h(y_{j}) \right) \left/ \sum_{j'}\sum_{i'}I_{i'}^{b}I_{j'}^{b}\right.\right],\label{eq:AUCstar}
\end{align}%
where $h= h_{\mathbf{X}^{\ast b}}$. We prove that (\APPENDIX) these two variants are not equivalent both asymptotically or under finite
$B$, exactly as their analogue of the error rate~\eqref{EqLOOerr} and~\eqref{EqStarErr}.

\subsubsection{Remarks}\label{sec:remarks}
Some concluding remarks are in order after the previous formalization. It has been proven that the two variants (the pooled and
partitioned) of each of CVK and CVKR are always equivalent; the two variants of the CVKM estimator are equivalent only asymptotically
as $M\rightarrow \infty$; and the two variants of the BS estimator are always different, even asymptotically as $B \rightarrow \infty$.
In addition, we prove that (\APPENDIX), as $M\rightarrow \infty$ the two variants of CVKM are equal to the CVKR. Said differently, The
CVKM and CVKR are equivalent asymptotically as $M\rightarrow\infty$, with a delayed conversion rate of the former.


  \section{Smoothness of Estimators}\label{sec:experimental-results}
\begin{figure}[t]\centering
  \includegraphics[height=2.5in]{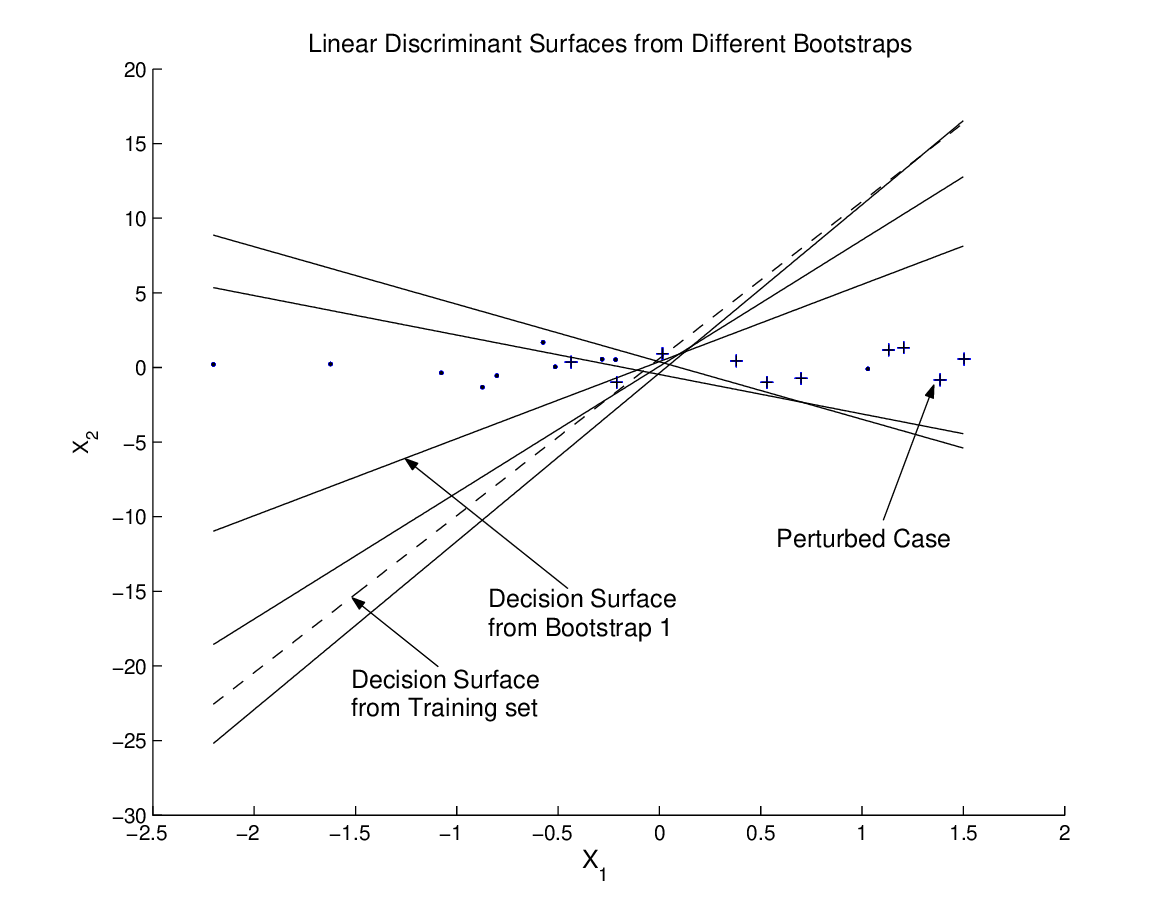}
  \caption{Different linear decision surfaces obtained by training on different BS replicates from the same training dataset.
    The first case from class 1 is chosen for perturbation. Changing a feature, e.g., $X_{1}$, has no change on the decision value of a
    single surface unless the case crosses that surface.}\label{fig12}
\end{figure}
\begin{figure}[t]\centering
  \includegraphics[height=2.5in]{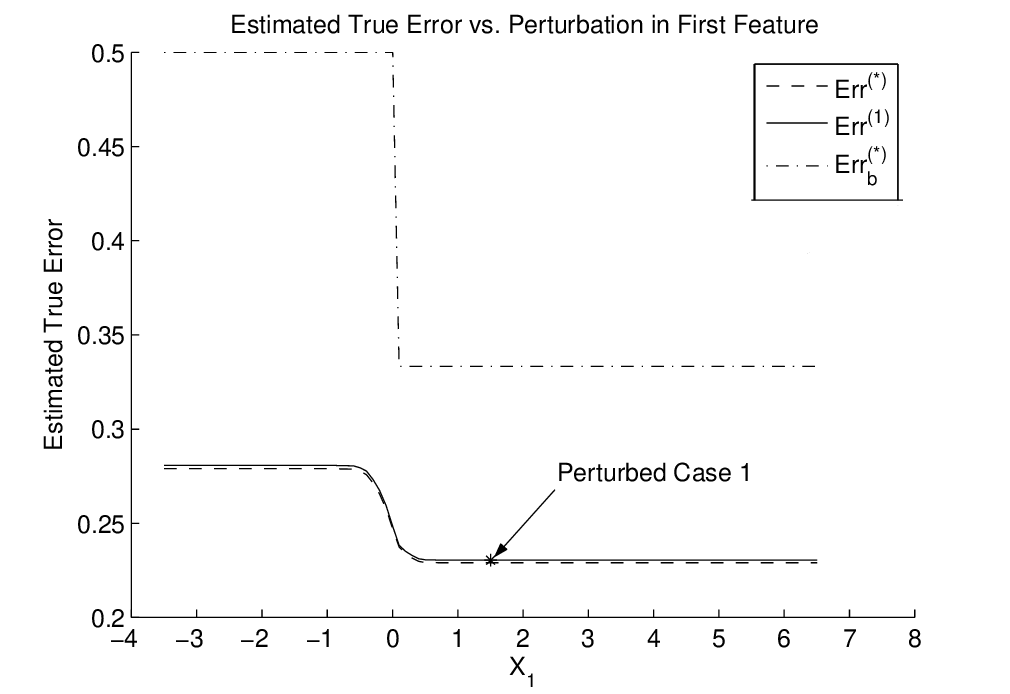}
  \caption{The two estimators $\widehat{\Err}^{_{\left( \ast\right) }}$, $\widehat{\Err}^{_{\left( 1\right) }}$, and the unsmooth component
    $\Err_{\mathbf{X}^{\ast b}}(\widehat {F}_{\varepsilon,i}^{_{\left( \ast\right) }})$ estimated after training on the first BS
    replicate. The estimated true error is plotted vs. change in the value of the first feature.}\label{fig13}
\end{figure}
In this section, we define theoretically the smoothness of an estimator, and explain experimentally the smoothness of the LOOB
estimator. Then we will be able to conclude which version or variant of the CV-based estimators is smooth and, hence, suitable for
estimating its standard error using the IF approach. The mathematics behind the current section is scattered in several places, which
we recently reviewed in~\cite{Yousef2022MachineLearning}, which is in turn summarized in the following paragraph for readers'
convenience. Consider the distribution $F$ of the data $\mathbf{X}$. The maximum likelihood estimator (MLE) of $F$ assigns a
probability of $1/n$ to each observation $x_i \in \mathbf{X}$, and is denoted $\hat{F}$. The BS replications $\mathbf{X}^{\ast}$ are
sampled from $\hat{F}$, and has a distribution $\hat{F}^{*}$. In addition, let $\hat {F}^{\left( \ast\right) }$ denote the distribution
of those observations in $\mathbf{X}$ that do not appear in the BS replication $\mathbf{X}^{\ast}$. The parenthesis notation $\left(
  \ast\right) $ refers to excluding from $\widehat{F}$, in the testing stage, the training cases $\mathbf{X}^{*}$ that were generated
by the BS replication. For better pedagogy, and WLOG, consider the error rate as the performance measure of interest rather than the
AUC. The variant $\widehat{\Err}^{_{\left( \ast\right) }}$ \eqref{EqStarErr} can be rewritten as%
\begin{subequations}\label{eq:Errstar}
  \begin{align}
    \widehat{\Err}^{_{\left(  \ast\right)  }} &= \MEAN_{\ast}\MEAN_{\hat{F}^{\left( \ast\right)  }}\left[  {Q\left( x_{i},\mathbf{X}^{\ast}\right)|\mathbf{X}^{\ast} }\right]\\
                                            &  =\MEAN_{\ast}{\Err_{\mathbf{X}^{\ast b}}(\widehat{F}^{\left( \ast\right)  })},
  \end{align}
\end{subequations}
where $\MEAN_{\ast}$ is approximated by the averaging $\frac{1}{B}\sum_b$ over a finite number of $B$ bootstraps, and
$\Err_{\mathbf{X}^{\ast b}}(\widehat{F}^{\left( \ast\right) })$ denotes the whole inner bracket of~\eqref{EqStarErr}. Applying the IF
approach to estimate the variance requires perturbing each observation $x_i$, with a probability measure $\epsilon$, to produce a new
distribution $\widehat{F} _{\varepsilon,i}$ than $\widehat{F}$, and study the effect of this variation on the estimator. This
probability perturbation of course propagates through to the probability masses of the BS replicates as well. Under the new
distribution $\widehat{F} _{\varepsilon,i}$, it can be shown that~\citep{Efron1992JackknifeAfter} the $b^{\text{th}}$ BS includes the
case $x_i$, $N_{i}^{b}$ times, with probability:%
\begin{equation}
  g_{\varepsilon,i}(b) =(1-\varepsilon)^{n}(1+\frac{n\varepsilon}{1-\varepsilon})^{N_{i}^{b}}(1/n)^{n}.\label{eq123}
\end{equation}
Then, the estimator $\widehat{\Err}^{_{\left( \ast\right) }}$, after perturbation, is evaluated as%
\begin{equation}
  \widehat{\Err}^{_{\left(  \ast\right)  }}(\widehat{F} _{\varepsilon,i}) =\sum\nolimits_{b}{g_{\varepsilon,i}(b)\,\Err_{\mathbf{X} ^{\ast b}}(\widehat{F}_{\varepsilon,i}^{\left(  \ast\right)  })}.\label{eq:Eq33PRL}
\end{equation}
The reader should note that if there is no perturbation, i.e., $\varepsilon=0$, then $\widehat{F} _{\varepsilon,i}$ is reduced to
$\widehat{F}$, $g_{\varepsilon,i}(b)$ is reduced to $(1/n)^n$, which is the probability of each BS replication, and (\ref{eq:Eq33PRL})
is reduced to the original averaging~\eqref{EqStarErr} over the bootstraps. Calculating the variance of $\widehat{\Err}^{_{\left(
      \ast\right) }}$ using the IF requires calculating the derivative $\partial \widehat{\Err}^{_{\left( \ast\right) }}(\widehat{F}
_{\varepsilon,i})/\partial\varepsilon$. If this derivative does not exist, i.e., a small change in an observation does not result in a
small change in the statistic, the latter is not ``smooth'', and hence its variance is not estimable using the IF approach.
Unfortunately, the inner component $\Err_{\mathbf{X} ^{\ast b}}(\widehat{F}_{\varepsilon,i}^{\left( \ast\right) })$ is not
differentiable, which will be explained by the following experiment.

Consider the very simple case where there are just two features and the classifier is the LDA, where the decision surface will be a
straight line in $\mathbf{R}^2$. \figurename~\ref{fig12} illustrates a sample generated from the two classes, together with the
decision surface obtained by training on this sample, along with the decision surfaces obtained from training on five BS replicates
(one at a time). Each decision surface trained on a BS replicate $\mathbf{X}^{\ast b}$, and tested on the sample cases not included in
the training, produces an estimate $\Err_{\mathbf{X}^{\ast b}}(\hat {F}_{\varepsilon,i}^{\left( \ast\right) })$, which is clearly
unsmooth. This is because the estimate does not change with a change in a feature value, e.g., $X_{1}$, unless this change allows the
observation to cross the decision surface.

On the contrary, the LOOB (\ref{EqLOOerr}), $\widehat{\Err}^{_{\left( 1\right) }}$, has an inner summation that is a smooth function.
This is so since any change in a sample case will cross many BS-based decision surfaces (some extreme violations to this fact may occur
under particular classifiers). For more illustration, with $B=1000$, the behaviour of $\widehat{\Err}^{_{\left( 1\right) }}$,
$\widehat{\Err}^{_{\left( \ast\right) }}$, and its inner component $\Err_{\mathbf{X}^{\ast
    b}}(\widehat{F}_{\varepsilon,i}^{\left(\ast\right) })$ are shown in~\figurename~\ref{fig13}. Therefore, although
$\widehat{\Err}^{_{\left( 1\right) }}$ and $\widehat{\Err}^{_{\left( \ast\right) }}$ almost give the same estimated value and both look
smooth, the former, in contrast to the latter, has a smooth inner summation, which makes it suitable for using the differential
operator of the influence function. To recap, training on many datasets (the BS replicates of the BS-based estimators) resulted in many
decision surfaces. Averaging the output of these decision surfaces, for each observation, results in a smooth and differentiable
estimator.

\medskip

Replacing the averaging of the BS replicates by the averaging of the CV partitions allows us immediately to conclude which CV-based
estimators are smooth and differentiable. It is clear that only the pooled variants~\eqref{EqCVKRa} of CVKR and~\eqref{EqCVKM} of CVKM
are the analogue to $\widehat{\Err}^{_{\left( 1\right) }}$. Therefore, for the same reason explained above, they are smooth and
differentiable, and hence eligible for estimating their variance using the IF approach. On the other hand, all the CV partitioned
variants (those denoted by $(*)$) are the analogue to $\widehat{\Err}^{_{(*) }}$, and therefore they are not differentiable. In
addition, it is also clear that, LOOCV and CVK are neither smooth nor differentiable because each observation is used for testing by
only a single trained classifier.




  \section{More Properties of Estimators}\label{sec:theoritical-critique}
In addition to the properties discussed in the previous two sections, we discuss here more properties in the light of some known
results in the literature (all the proofs are deferred to the~\APPENDIX).

Although all estimators train on the same training set of size $n$, because of their different resampling mechanisms, the ``effective''
training set size is different. It is clear that all CV versions train on exactly $n-n/K$ observations in each training iteration,
including the CVN with $K=n$. In addition they all give the same chance to all observations to appear in the training sets (including
an outlier that no other observation in the training set is close to it). On the contrary, the BS-based estimators sample, with
replacement, from the MLE distribution. This gives more chance to close-to-each-other observations to appear in the bootstraps than
other distant-from-each-other observations, which mimics the true distribution. It is theoretically argued that the effective number of
observations included in the BS replicates is $.632n$, which introduces more bias to the estimator than CV. Experimentally, it appears
that the bootstrap is supported on almost half of the observations, i.e., only $0.5n$. We show that if we look at the bootstrap
resampling mechanism as sampling with replacement without ordering, the BS will be indeed supported on almost $0.5n$
(Corollary~\ref{COR632or5}).

We next compare the pooled variant of all estimators, CVN, CVK, CVKR, CVKM, LOOB,
i.e.,~\eqref{eq:ErrCVN},~\eqref{EQCVK},~\eqref{EqCVKRa},~\eqref{EqCVKM},~\eqref{EqLOOerr} for error rate, or their counterparts for AUC
(all these estimators are listed side-by-side in the~\APPENDIX~for readers' convenience). CVN and CVK are the only versions that do not
average over training sets for every observation. E.g., the error from testing on $x_{i}$ is produced solely by training once on
$\mathbf{X}_{\left( i\right) }$ or $\mathbf{X}_{\left( \{\mathcal{K}\left( i\right) \}\right) }$, respectively. Hence, they can be seen
as if they were designed to estimate the conditional performance $Err_{\mathbf{X}} = {\E_o}Q\left(x_o,\mathbf{X}\right)$ of the
classification rule; i.e., the performance conditional on the training set $\mathbf{X}$. However, CVKR and LOOB differ from CVN and CVK
in that both average over many training sets (MC or bootstraps, respectively) to produce one error estimate per observation, i.e.,
$\left. \sum_{m}Q\left( x_{i},\mathbf{X}_{\left( \left\{ \mathcal{K}_{m}\left( i\right) \right\} ,m\right) }\right) \right/ M$, and
$\left. \sum_{b}I_{i}^{b}Q\left( x_{i},\mathbf{X}^{\ast b}\right) \right/ \sum_{b}I_{i}^{b}$ respectively. Hence, they can be seen as
if they are designed to estimate the unconditional performance $Err =
{\E_{\mathbf{X}}}Err_{\mathbf{X}}={\E_o}{\E_{\mathbf{X}}}Q\left(x_o,\mathbf{X}\right)$ of the classification rule; i.e., the expected
performance, where the expectation is taken over the population of training sets of the same size $n$. Therefore,
\cite{Efron1997ImprovementsOnCross} call LOOB, and for the same reason we can call each of CVKR and CVKM, a smooth version of CVN.
Although both CVN and CVK look, in theory, like estimating the conditional performance, whereas the CVKR (or the slower estimator CVKM)
and LOOB look like estimating the mean performance, all in practice do estimate the expected performance! This is a consequence of the
known phenomenon of ``weak correlation'' that was reported and analyzed several years ago,
e.g.,~\citep{Efron1983EstimatingTheError,Efron1997ImprovementsOnCross,Zhang1995AssessingPrediction}. We provide below a qualitative
interpretation for this phenomenon that helps practitioners to understand this phenomenon. Denote the true performance of the
classification rule, conditional on the training set $\mathbf{X}$ (whether $Err_{\mathbf{X}}$ or $AUC_{\mathbf{X}}$), by
$S_{\mathbf{X}}$, the unconditional performance by $\E_{\mathbf{X}}S_{\mathbf{X}}$, and an estimator of either of them by
$\widehat{S}_{\mathbf{X}}$. For easier notation we can unambiguously drop the subscript $\mathbf{X}$ and it is then straightforward to
decompose the MSE, normalized by the standard deviations, as:%
\begin{align}
  \frac{\MSE\left(\widehat{S}, S\right)}{\sigma_S \sigma_{\widehat{S}}} &= \frac{\MSE\left(\widehat{S}, \E S\right)}{\sigma_S \sigma_{\widehat{S}}} + \frac{\sigma_S}{\sigma_{\widehat{S}}} - 2\rho_{\widehat{S}S}.\label{eq:14}
\end{align}
This equation relates four crucial components to each other:
\begin{itemize}[partopsep=0in,parsep=0in,topsep=0.1in,itemsep=0.0in,leftmargin=0.2in]
\item $\MSE(\widehat{S}, S)\bigl/\sigma_S \sigma_{\widehat{S}}$, the normalized MSE of $\widehat{S}$, if we see it as an estimator of the
  conditional performance $S$.

\item $\MSE(\widehat{S}, \E S)\bigl/\sigma_S \sigma_{\widehat{S}}$, the normalized MSE of $\widehat{S}$, if we see it as an estimator of
  the expected performance $\E S$ (and therefore called MSE around the mean).

\item $\sigma_S \bigl/ \sigma_{\widehat{S}}$, the standard deviation ratio between $S$ and $\widehat{S}$.

\item $\rho_{\widehat{S}S}$, the correlation coefficient between $S$ and $\widehat{S}$.
\end{itemize}
From~\eqref{eq:14}, an estimator $\widehat{S}$ is a good candidate to estimate $S$ than $\E S$ if its $\MSE\left(\widehat{S}, S\right)$
is less than its $\MSE\left(\widehat{S}, \E S\right)$. Then, it is the responsibility of the correlation coefficient
$\rho_{\widehat{S}S}$ to be high enough to cancel $\sigma_S \bigl/ \sigma_{\widehat{S}}$ and a portion of $\MSE\left(\widehat{S}, \E
  S\right)$. Unfortunately, this is not the case in almost all experiments reported in the literature. We provide our illustrative
experiments in the~\APPENDIX.

It is important to emphasize that these estimators were long standing in the literature and have been numerically compared in some
publications, including
\cite{Efron1997ImprovementsOnCross,Sahiner2001ResamplingSchemes,Yousef2005EstimatingThe,Sahiner2008ClassifierPerformance}. Conducting a
new comprehensive comparative study, which considers new distributions and real datasets than those considered in the literature, is
quite important to conclude the overall winner, in terms of the MSE.


  \section{Conclusion and Future Work}\label{sec:conclusion}
We analyzed different versions (and different variants of each version) of CV-based estimators that estimate either the error rate or
the AUC of a classification rule. In addition, we analyzed the BS-based estimator because of its strong connection to the CV and its
known smoothness property that allowed, in the past, using the IF approach for estimating its variance. We can summarize the findings
of the present articles as a set of properties of these estimators as follows: (1) The CVN has only one variant, and the variants of
CVK and CVKR are equivalent. The two variants of CVKM are equivalent to each other, and interestingly to the CVKR as well, only
asymptotically with the number of MC trials (i.e., as $M \rightarrow \infty$). However, the two variants of the BS estimator are not
equivalent under either finite or infinite number of BS replicates $B$. (2) We proved mathematically and illustrated experimental that
the only smooth and differentiable versions of the CV-based estimators are the pooled variants of the CVKR and CVKM, which allows
estimating their variance using the IF approach. (3) In addition, some miscellaneous remarks that connect the present article to some
already-known results in the literature are in order. Although all of these estimators train on a similar size of training dataset, the
effective training set size (because of different sampling mechanisms) are different, which contributes to a different bias to each
estimator. It was clear that both CVN and CVK look like estimating the conditional performance of a classification rule; whereas,
similar to the BS estimator, both CVKR and CVKM look like estimating the mean (over different training sets) performance of the
classification rule. However, in terms of MSE, all estimators in practice almost estimate the mean performance due to the known
phenomenon of weak correlation between each of these estimators and the conditional performance.

In retrospect, to be able to provide a final advice for practitioners, two future works are still needed, which is specially true after
the new era of deep learning and the very complex classifiers they produce: (1) applying the IF approach to either the CVKR or CVKM to
estimate their variance, as was done very successfully for the BS-based estimators, and (2) conducting a very comprehensive
computational study, which supports many distributions, real datasets, classifier architectures, etc., to compare the relative
accuracies of these estimators.


  \section{Acknowledgment}\label{sec:acknoledgment}
The author is grateful to the U.S. Food and Drug Administration (FDA) for funding an earlier stage of this project, circa 2008; and to
both, Brandon Gallas and Weijie Chen, of the FDA, for reviewing an earlier stage of this manuscript when it was an internal report; and
to Kyle Myers, for her continuous support. Special thanks and gratitude, in his memorial, to Dr. Robert F. Wagner the supervisor and
teacher, or Bob Wagner the big brother and friend, who passed away before seeing this manuscript published.


  \small
  \let\oldbibliography\thebibliography
  \renewcommand{\thebibliography}[1]{%
  \oldbibliography{#1}%
  \setlength{\itemsep}{0pt}%
  }
  \putbib[booksIhave,publications]
\end{bibunit}

\clearpage
 
\begin{bibunit}[model2-names]
  \normalfont
  \onecolumn
  \section{Appendix}\label{sec:appendix}
\subsection{History}
The early motivation behind our present article started since our contribution in the MAQC project, which was formally launched by
scientists at the FDA's National Center for Toxicological Research (NCTR), Jefferson, Arkansas, in response to the FDA Critical Path
Initiative. MAQC was a group of over 100 members from industry, academia, and US government working on methods for developing
predictive models that use high-dimensional microarray (DNA chips) data to classify patients into low- or high-risk with respect to
getting a specified kind of cancer. In particular, phase II of the MAQC project (MAQCII) aimed at providing the best practices to
design DNA microarray classification methods. This is, in many cases, an ill-posed problem since the dimensionality (number of genes)
is in thousands whereas the number of observations (patients) is in hundreds. Classifiers designed for such a problem can easily suffer
from large variability (unstability); hence they may not generalize, and results obtained from the studies are fragile. Our main
conclusion of the project was published in Nature biotechnology \citep{Shi2010MAQCII}, in addition to other several publications each
was an in-depth treatment of a technical aspect of the project. One of these publications is \cite{Chen2012UncertEst}, in which we
discussed the uncertainty (standard error) estimation associated with the assessment of classification models under the scarcity of
data, where we applied the techniques developed earlier in \cite{Yousef2005EstimatingThe, Yousef2006AssessClass} to assess classifiers
in two different paradigms. In Paradigm I \citep{Yousef2005EstimatingThe}, the classifier is assessed from only one available dataset,
where resampling must take place to train and test on two different datasets in each resampling iteration. However, the resampling
approach is a ``smooth'' version of the bootstrap (BS) rather than the Cross Validation (CV); therefore, the variance of the former can
be estimated almost unbiasedly using the Influence Function (IF) approach derived in \citep{Efron1997ImprovementsOnCross} for the error
rate and extended in \cite{Yousef2005EstimatingThe} for the AUC. This is as opposed to the CV estimators where only ad-hoc estimators
of their standard error are available that ignore the covariance structure among folds, and no unbiased estimator of their variance
exists \citep{Bengio2004NoUnbiasedEstKCV}. Paradigm II, that was proposed by \cite{Yousef2006AssessClass} and extended later in
\cite{Chen2012ClassVar}, assumes that the assessment must be carried over two independent and sequestered datasets, one for training
and one for testing. This protocol is mandatory in some regulatory agencies, e.g., the FDA. In this paradigm, the performance estimate
for either the error rate or the AUC, along with their variance estimation, are obtained via the Uniform Minimum Variance Unbiased
Estimator (UMVUE) derived from the $U$-statistic theory~\citep{Randles1979IntroductionTo}. Part of the scientific debate, in particular
among colleagues of the project that time, or in general in the whole field, is on the assessment strategy and the several versions of
estimators that practitioners do (or should) use under paradigm I, and which estimator is ``better'' and in which ``sense''?

\subsection{Experimental Results}\label{sec:experimental-restuls}
\begin{figure}[bh]\centering
  \includegraphics[width=3in]{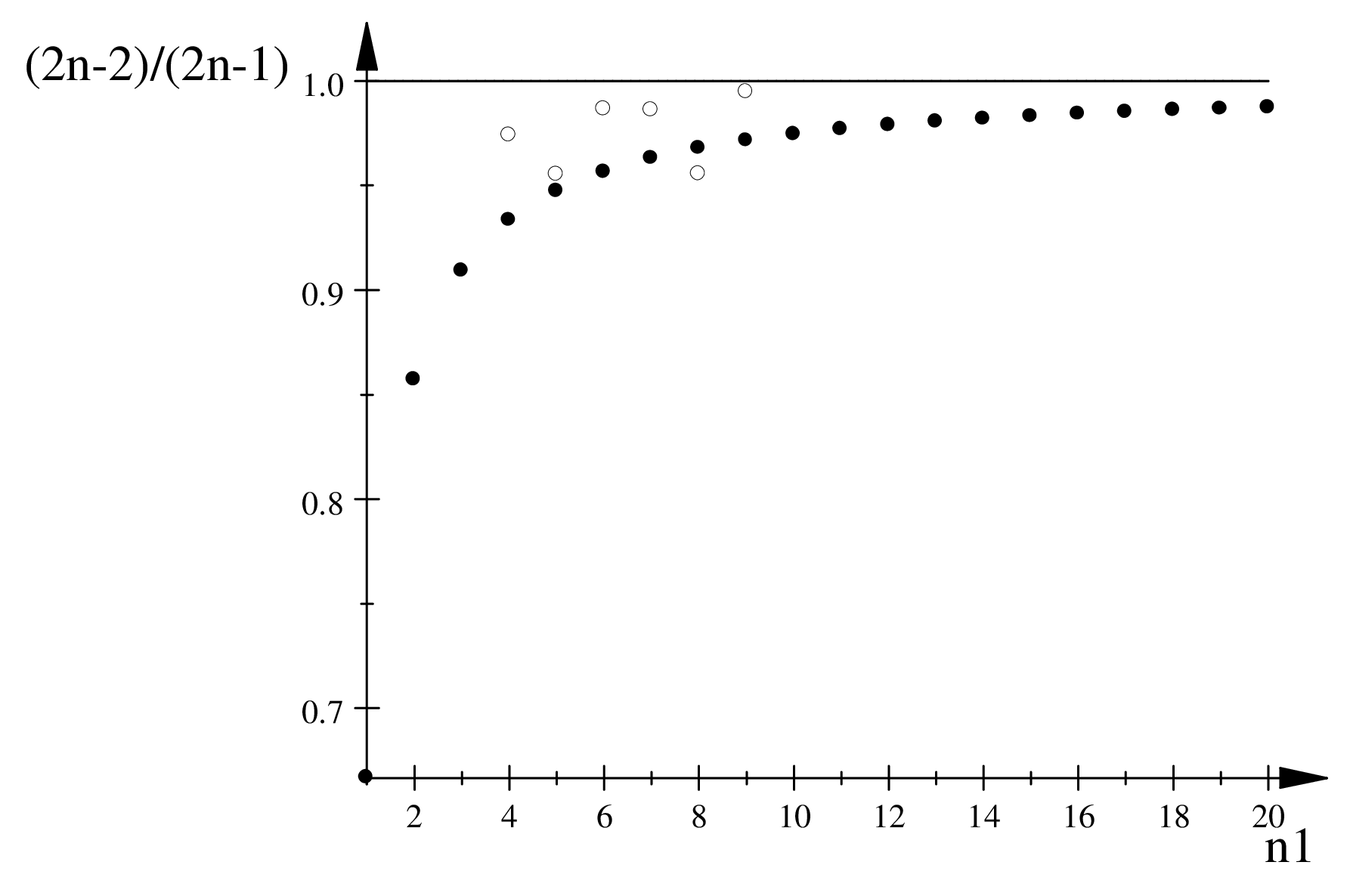}
  \caption{The ratio between the two variants of the BS vs. sample size}\label{FIGratio}
\end{figure}
The ratio between the two estimators, $\frac{2n-2}{2n-1}$, is plotted in Figure~\ref{FIGratio} vs. $n_{1}~(=n_{2})$, where
$n~(=n_{1}+n_{2})$ is the number of observations in the two classes and $n_{1}$ and $n_{2}$ are the number of observations in each
class. Then, the two estimators are identical only for large $n$. The white circles in the figure represent the ratio between the two
estimators for LDA classifier, using single dimensionality, 0 and 1 means, and unit variance. The number of bootstraps is $200$. There
is difference between the theoretical result and the simulation since the simulation assumes sampling with replacement \textbf{with
  ordering}.
\begin{table}[t]\centering
  \small
  \caption{A sample from a large set of experiments that shows true performance $S$, apparent performance on the same training set
    $\overline{S}$, and estimator $\widehat{S}$ from equations \eqref{eq:6} or \eqref{eq:7}. (The typical values in this table are for
    $S = AUC_{\mathbf{X}}$ and $\widehat{S} = \widehat{AUC}_{\mathbf{X}}^{(*)}$).}\label{tab1}
  \resizebox{0.5\columnwidth}{!}{\begin{tabular}{clllllr}
  \toprule
  $\widehat{S}$ & $\E$ & $\sigma$ & $\RMS(\widehat{S}, S)$ & $\RMS(\widehat{S}, \E S)$ & $\rho_{\widehat{S}, S}$ & $n$\\
  \midrule
  \multicolumn{1}{l}{$S$} & 0.6181 & 0.0434 & 0 & 0.0434 & 1.0000 & \\
  \multicolumn{1}{l}{$\overline{S}$} & 0.8897 & 0.0475 & 0.2774 & 0.2757 & 0.2231 & 20\\
  \multicolumn{1}{l}{$\widehat{S}$} & 0.5914 & 0.0947 & 0.0973 & 0.0984 & 0.2553 & \\\hline

  \multicolumn{1}{l}{$S$} & 0.6231 & 0.0410 & 0 & 0.0410 & 1.0000 & \\
  \multicolumn{1}{l}{$\overline{S}$} & 0.8788 & 0.0499 & 0.2615 & 0.2606 & 0.2991 & 22\\
  \multicolumn{1}{l}{$\widehat{S}$} & 0.5945 & 0.0947 & 0.0956 & 0.0990 & 0.2993 & \\\hline

  \multicolumn{1}{l}{$S$} & 0.6308 & 0.0400 & 0 & 0.0400 & 1.0000 & \\
  \multicolumn{1}{l}{$\overline{S}$} & 0.8656 & 0.0471 & 0.2406 & 0.2395 & 0.2833 & 25\\
  \multicolumn{1}{l}{$\widehat{S}$} & 0.5991 & 0.0865 & 0.0897 & 0.0922 & 0.2946 & \\\hline

  \multicolumn{1}{l}{$S$} & 0.6359 & 0.0358 & 0 & 0.0358 & 1.0000 & \\
  \multicolumn{1}{l}{$\overline{S}$} & 0.8554 & 0.0472 & 0.2253 & 0.2246 & 0.2747 & 28\\
  \multicolumn{1}{l}{$\widehat{S}$} & 0.6035 & 0.0840 & 0.0874 & 0.0901 & 0.2904 & \\\hline

  \multicolumn{1}{l}{$S$} & 0.6469 & 0.0343 & 0 & 0.0343 & 1.0000 & \\
  \multicolumn{1}{l}{$\overline{S}$} & 0.8419 & 0.0439 & 0.2010 & 0.1999 & 0.2434 & 33\\
  \multicolumn{1}{l}{$\widehat{S}$} & 0.6170 & 0.0750 & 0.0792 & 0.0807 & 0.2746 & \\\hline

  \multicolumn{1}{l}{$S$} & 0.6571 & 0.0308 & 0 & 0.0308 & 1.0000 & \\
  \multicolumn{1}{l}{$\overline{S}$} & 0.8246 & 0.0431 & 0.1735 & 0.1730 & 0.2923 & 40\\
  \multicolumn{1}{l}{$\widehat{S}$} & 0.6244 & 0.0711 & 0.0753 & 0.0783 & 0.3185 & \\\hline

  \multicolumn{1}{l}{$S$} & 0.6674 & 0.0271 & 0 & 0.0271 & 1.0000 & \\
  \multicolumn{1}{l}{$\overline{S}$} & 0.8091 & 0.0406 & 0.1473 & 0.1474 & 0.3517 & 50\\
  \multicolumn{1}{l}{$\widehat{S}$} & 0.6357 & 0.0654 & 0.0690 & 0.0727 & 0.3534 & \\\hline

  \multicolumn{1}{l}{$S$} & 0.6808 & 0.0217 & 0 & 0.0217 & 1.0000 & \\
  \multicolumn{1}{l}{$\overline{S}$} & 0.7946 & 0.0355 & 0.1195 & 0.1192 & 0.2499 & 66\\
  \multicolumn{1}{l}{$\widehat{S}$} & 0.6533 & 0.0546 & 0.0602 & 0.0611 & 0.2451 & \\\hline

  \multicolumn{1}{l}{$S$} & 0.6965 & 0.0158 & 0 & 0.0158 & 1.0000 & \\
  \multicolumn{1}{l}{$\overline{S}$} & 0.7772 & 0.0312 & 0.0860 & 0.0866 & 0.3596 & 100\\
  \multicolumn{1}{l}{$\widehat{S}$} & 0.6738 & 0.0454 & 0.0483 & 0.0507 & 0.3422 & \\\hline

  \multicolumn{1}{l}{$S$} & 0.7141 & 0.0090 & 0 & 0.0090 &1.0000 & \\
  \multicolumn{1}{l}{$\overline{S}$} & 0.7573 & 0.0228 & 0.0487 & 0.0489 & 0.2277 & 200\\
  \multicolumn{1}{l}{$\widehat{S}$} & 0.6991 & 0.0298 & 0.0327 & 0.0334 & 0.2288 & \\
  \bottomrule
\end{tabular}

}
\end{table}
We next elaborate experimentally on the concept of weak correlation discussed by a simple toy example. We generate a five dimensional
($p=5$) multinormal 2-class data with identity covariance matrices with mean vectors of $\mathbf{0}$ and $c \mathbf{1}$. It is known
that the Mahalanobis distance between the two classes, defined as $\Delta=\left[
  (\mu_{1}-\mu_{2})'\Sigma^{-1}(\mu_{1}-\mu_{2})\right]^{1/2}$, is given by $\Delta = c^{2}p$ for the multinormal distributions.
Therefore, we adjust $c$ to keep a reasonable inter-class separation of $\Delta = 0.8$. When the classifier is trained, it will be
tested on a pseudo-infinite test set, here 1000 cases per class, to obtain a very good approximation to the true AUC for the classifier
trained on this very training data set. This is called a single realization or a Monte-Carlo (MC) trial. Many realizations of the
training data sets with same $n$ are generated over MC simulation to study the mean and variance of the AUC for the Bayes classifier
under this training set size. The number of MC trials used is 1000. \tablename~\ref{tab1} provides all quantities of the
decomposition~\eqref{eq:14} of these experiments under different values of training sample size $n$. It is obvious from the values that
$\RMS(\widehat{S}, S)$ and $\RMS(\widehat{S}, \E S)$ are very close to each other because the quantity $\sigma_S \bigl/
\sigma_{\widehat{S}} - 2\rho_{\widehat{S}S} \simeq 0.413 - 2 \times 0.290 = -0.167$ (on average over the 10 experiments shown in the
table); and in some cases, e.g., the first experiment, it goes as low as $-0.052$. For short, the correlation between $\widehat{S}$ and
$S$ is weak to cast $\widehat{S}$ as an estimate to $S$ even though it is designed to estimate it!

\subsection{Lemmas and Proofs}\label{sec:lemmas-proofs}
The pooled variant of each error rate estimator (CVN~\eqref{eq:ErrCVN}, CVK~\eqref{EQCVK}, CVKR~\eqref{EqCVKRa}, CVKM~\eqref{EqCVKM},
and LOOB~\eqref{EqLOOerr}), are re-listed below respectively:%
\begin{subequations}\label{eq:6}
  \begin{align}
    \widehat{Err}^{\left( CVN\right) } & =\frac{1}{n}\sum_{i=1}^{n}\, \Biggl[ Q\left( x_{i},\mathbf{X}_{\left( i\right) }\right) \Biggr].\label{eq:1}\\
    \widehat{Err}^{\left( CVK\right) } & =\frac{1}{n}\sum_{i=1}^{n}\, \Biggl[Q\left( x_{i},\mathbf{X}_{\left( \{\mathcal{K}\left( i\right) \}\right) }\right) \Biggr].\label{eq:2}\\
    \widehat{Err}^{\left( CVKR\right) } & =\frac{1}{n}\sum_{i=1}^{n}\left[ \left.  \sum_{m}Q\left( x_{i},\mathbf{X}_{\left( \left\{ \mathcal{K}_{m}\left( i\right) \right\} ,m\right) }\right) \right/ M\right].\label{eq:3}\\
    \widehat{Err}^{\left( CVKM\right) } & =\frac{1}{n}\sum_{i=1}^{n}\left[ \left.  \sum_{m}I_{i}^{m}Q\left( x_{i},\mathbf{X}_{\left( \left\{ 1\right\} ,m\right) }\right) \right/ \sum_{m}I_{i}^{m}\right].\label{eq:4}\\
    \widehat{Err}^{\left(  1\right)  }  &  =\frac{1}{n}\sum_{i=1}^{n}\left[\left.\sum_{b}I_{i}^{b}Q\left(  x_{i},\mathbf{X}^{\ast b}\right) \right/ \sum_{b}I_{i}^{b}\right].\label{eq:5}
  \end{align}
\end{subequations}

The pooled variant of each of AUC estimators (CVN \eqref{eq:AUCCVN}, CVK~\eqref{eq:CVK-AUC}, CVKR~\eqref{eq:CVKR-AUC},
CVKM~\eqref{eq:CVKM-AUC}, and LOOB~\eqref{eq:BS-AUC}), are listed below respectively.%
\begin{subequations}\label{eq:7}
  \begin{align}
    \widehat{AUC}^{\left( CVN\right) } & =\frac{1}{n_{1}n_{2}}\sum_{j=1} ^{n_{2}}\sum_{i=1}^{n_{1}}\, \Biggl[ \psi\left( h(x_i), h(y_j)\right)\Biggr], &&h= h_{\mathbf{X1}_{\left( i\right) }\mathbf{X2}_{\left( j\right) }}.\label{eq:8}\\
    \widehat{AUC}^{\left( CVK\right) } & =\frac{1}{n_{1}n_{2}}\sum_{j=1} ^{n_{2}}\sum_{i=1}^{n_{1}}\, \Biggl[\psi\left( h(x_i), h(y_j)\right) \Biggr], && h= h_{\mathbf{X1}_{\left( \left\{ \mathcal{K}_{1}(i)\right\} \right) }\mathbf{X2}_{\left( \left\{ \mathcal{K}_{2}(j)\right\} \right) }} .\label{eq:9}\\
    \widehat{AUC}^{\left( CVKR\right) } & =\frac{1}{n_{1}n_{2}}\sum _{j=1}^{n_{2}}\sum_{i=1}^{n_{1}}\left[ \left.  \sum_{m}\psi\left( h(x_i), h(y_j)\right) \right/ M\right], && h= h_{\mathbf{X1}_{\left( \left\{ \mathcal{K}_{1m}(i)\right\} \right) }\mathbf{X2}_{\left( \left\{ \mathcal{K}_{2m}(j)\right\} \right) }} .\label{eq:10}\\
    \widehat{AUC}^{\left(  CVKM\right)  } &  =\frac{1}{n_{1}n_{2}}\sum_{j=1}^{n_{2}}\sum_{i=1}^{n_{1}}\left[  \left.  \sum_{m}I_{j}^{m}I_{i}^{m}\psi\left( h(x_i), h(y_j)\right)  \right/  \sum_{m}I_{i}^{m}I_{j}^{m}\right], && h =   h_{\mathbf{X1}_{\left(\{1\},m\right)}\mathbf{X2}_{\left(  \left\{  1\right\},m\right)  }}.\label{eq:11}\\
    \widehat{AUC}^{\left( 1,1\right) } & =\frac{1}{n_{1}n_{2}}\sum_{j=1} ^{n_{2}}\sum_{i=1}^{n_{1}}\left[ \left.  \sum_{b}I_{i}^{b}I_{j}^{b}\psi\left( h(x_i), h(y_j)\right) \right/ \sum_{b}I_{i}^{b}I_{j}^{b}\right], && h= h_{\mathbf{X}^{\ast b}}.\label{eq:12}
  \end{align}
\end{subequations}

\begin{lemma}
  \label{LEMIiXiConvergence}Consider the sequence of r.v. $\left\{ X_{i}\right\} $, where $\sum_{1}^{n}\frac{x_{i}}{n}\overset{a.s.}%
  {\rightarrow}S$, and $S$, in general, is a r.v.. Consider, as well, both, the sequence $\left\{ I_{i}\right\} $ of independent events,
  where $I_{i}\sim Ber(p)$, and the sequence $\left\{ a_{i}\right\} $ of discrete i.i.d. r.v. with mean $\operatorname*{E}\left[
    a_{1}\right] $. Then,
  \begin{enumerate}
    \Item
    \begin{align}
      \sum_{1}^{n}\frac{I_{i}x_{i}}{\Sigma_{i}I_{i}}\overset{a.s.}{\rightarrow}S.
    \end{align}

    \Item
    \begin{align}
      \sum_{1}^{n}\frac{I_{i}x_{i}}{n}\overset{a.s.}{\rightarrow}pS.
    \end{align}

    \Item
    \begin{align}
      \sum_{1}^{n}\frac{a_{i}x_{i}}{n}\overset{a.s.}{\rightarrow}\operatorname*{E}%
      \left[  a_{1}\right]  S.
    \end{align}
  \end{enumerate}
\end{lemma}
\begin{proof}
  \begin{enumerate}
  \item Consider the set$\mathbb{\ }\mathcal{J=}\left\{ j:I_{j}=1\right\} $; then $\sum_{1}^{n}\frac{I_{i}x_{i}}{\Sigma_{i}I_{i}}$ can be
    rewritten as $\sum_{1}^{N}\frac{x_{j}}{N}$, where $N=\left\vert \mathcal{J}\right\vert $. Then what remains is showing that $N$ is
    infinite (with the care that $N$ is a r.v.), which leads to that the summation converges a.s. by assumption. Since
    $\sum_{i=1}^{\infty}P\left( I_{i}\right) =\infty$, then by Borel-Cantelli Lemma $P\left( I_{i}~i.o.\right) =1$; or said
    differently, the set on which $\left\vert \mathcal{J}\right\vert $ is not infinite is a null set; hence the summation
    $\sum_{1}^{N}\frac{x_{j}}{N}$ does not converge (because of the reason that $N$ is finite) only on a set with measure zero; hence
    $\sum _{1}^{n}\frac{I_{i}x_{i}}{\Sigma_{i}I_{i}}\overset{a.s.}{\rightarrow}S$.

  \item We have by assumption $\sum_{1}^{n}\frac{x_{i}}{n}\overset {a.s.}{\rightarrow}S$, then $\sum_{1}^{n}\frac{\left(
        x_{i}-S\right)}{n}=\sum_{1}^{n}\frac{x_{i}}{n}-S\overset{a.s.}{\rightarrow}0$; and so does $\sum_{1}^{n}\frac{I_{i}\left(
        x_{i}-S\right)}{\Sigma_{i}I_{i}}$ by the first part of the Lemma. Since%
    \begin{align}
      \left\vert \frac{1}{n}\sum_{1}^{n}I_{i}\left(  x_{i}-S\right)  \right\vert \leq\left\vert \frac{1}{\sum_{i}I_{i}}\sum_{1}^{n}I_{i}\left(  x_{i}-S\right)\right\vert,
    \end{align}
    then,%
    \begin{align}
      \left\{  \omega:\left\vert \frac{1}{n}\sum_{1}^{n}I_{i}\left(  x_{i}-S\right)\right\vert >\varepsilon\right\}  \subseteq\left\{  \omega:\left\vert \frac{1}{\sum_{i}I_{i}}\sum_{1}^{n}I_{i}\left(  x_{i}-S\right)  \right\vert>\varepsilon\right\}.
    \end{align}
    But the set of the r.h.s. converges to a null set (by a.s. convergence proved above); hence the l.h.s. converges too to a null set,
    which concludes that%
    \begin{align}
      \frac{1}{n}\sum_{1}^{n}I_{i}\left(  x_{i}-S\right)  \overset{a.s.}%
      {\rightarrow}0.
    \end{align}
    Then%
    \begin{align}
      \frac{1}{n}\sum_{1}^{n}I_{i}x_{i}-S\sum_{1}^{n}\frac{I_{i}}{n}\overset{a.s.}{\rightarrow}0.
    \end{align}
    But since $I_{i}$s are i.i.d. then $\frac{1}{n}\sum_{1}^{n}I_{i}\overset{a.s.}{\rightarrow}p$ by SLLN, and by Slutsky's theorem%
    \begin{align}
      \frac{1}{n}\sum_{1}^{n}I_{i}x_{i}\overset{a.s.}{\rightarrow}pS.
    \end{align}

  \item For any discrete r.v. $a$ taking the values $a^{j}$ with probability $p^{j},~j=1,\dots k$ it can be represented as
    $a=\sum_{j}a^{j}I_{\left\{ \omega:a\left( \omega\right) =a^{j}\right\} }$, where the sets $\left\{ \omega:a(\omega)=a^{j}\right\} $
    partition $\Omega$. For simpilicity, we can write $a=\sum_{j}a^{j}I^{j}.$Then $a_{i}=\sum_{j}a^{j}I_{i}^{j}$, where $I_{i}^{j}$, by
    running the index $i$ and keeping $j$ fixed, are i.i.d. $Ber(p^{j})$. Then%
    \begin{align}
      \frac{1}{n}\sum_{i=1}^{n}a_{i}x_{i}  &  =\frac{1}{n}\sum_{i=1}^{n}\sum
                                         _{j=1}^{k}a^{j}I_{i}^{j}x_{i}\\
                                       &  =\sum_{j=1}^{k}a^{j}\frac{1}{n}\sum_{i=1}^{n}I_{i}^{j}x_{i}.
    \end{align}
    But we have just proved that $\frac{1}{n}\sum_{1}^{n}I_{i}^{j}x_{i}%
    \overset{a.s.}{\rightarrow}p^{j}S$. Then, again by Slutsky's%
    \begin{align}
      \frac{1}{n}\sum_{i=1}^{n}a_{i}x_{i}  &  \overset{a.s.}{\rightarrow}\sum_{j=1}^{k}a^{j}p^{j}S =\operatorname*{E}\left[  a_{1}\right]  S.
    \end{align}
  \end{enumerate}
\end{proof}

\begin{lemma}
  For a BS replication from $n$ observations (with sampling with replacement and without ordering) the quantity $a_{b}=$
  $\sum_{i}I_{i}^{b}$, where $I_{i}^{b}=1$ if the $i^{\text{th}}$ observation does not appear in the replication, has the following
  properties. (The third point will not be needed later in the analysis but is provided for mathematical interest and completeness)
  \begin{enumerate}
    \Item
    \begin{align}
      \Pr\left[  a_{b}=k\right]  =\frac{\tbinom{n}{k}\tbinom{n-1}{k}}{\tbinom
      {2n-1}{n}},~0\leq k\leq n-1,
    \end{align}

    \Item
    \begin{align}
      \operatorname*{E}a_{b}=\frac{n\left(  n-1\right)  }{2n-1}.
    \end{align}

    \Item
    \begin{align}
      \operatorname*{E}\frac{1}{1+a_{b}}=\frac{2}{n+1}.
    \end{align}
  \end{enumerate}
\end{lemma}
\begin{proof}
  \begin{enumerate}
  \item Bootstrapping a replication can be seen as sampling with replacement without ordering; there are $\tbinom{N+K-1}{K}$ different
    combinations for sampling $K$ from $N$ with replacement without ordering; If all of the observations appear in a replication, i.e.,
    $I_{i}^{b}=1~\forall$ $i$, then $a_{b}=0$; and if only one observation appears $n~$times then $a_{b}=n-1$. It follows that $0\leq
    a_{b}\leq n-1$. In general, sampling $m$ from $n$ with replacement without ordering will let $a_{b}$ take the value $k$. This $k$
    can be choosen by $\binom{n}{k}$. The remaining $n-k$ will appear as is in the $m$ sampled observations. The rest, i.e., $m-(n-k)$,
    will be sampled from the $n-k$. Hence, $N=n-k$ and $K=m+k-n$ and the probability of having $a_{b}=k$ is given by%
    \begin{align}
      \frac{\binom{n}{k}\binom{(n-k)+(m+k-n)-1}{m+k-n}}{\binom{n+m-1}{m}}%
      =\frac{\binom{n}{k}\binom{m-1}{k+m-n}}{\binom{m+n-1}{m}}.\label{EQgenPMF}
    \end{align}

    Bootstrapping is a special case, where $m=n$. What remains is proving that the probabilities add to $1$.
    \begin{align}
      \left(  1+x\right)  ^{m} &  =\left(  1+x\right)  ^{n}\left(  1+x\right)
                                 ^{m-n},~\forall~x,~0\leq n\leq m\\
                               &  =\sum_{i=0}^{n}\sum_{j=0}^{m-n}\tbinom{n}{i}x^{i}\tbinom{m-n}{j}%
                                 x^{j}.\label{EQMatchCoef}
    \end{align}
    $\newline$Then the coefficient of $x^{k},~0\leq k\leq m$, should match for both sides, which leads to%
    \begin{align}
      \binom{m}{k}=\sum_{i=0}^{k}\binom{n}{i}\binom{m-n}{k-i},
    \end{align}
    where $\binom{y}{x}=0$ for $x>y$. If we set $m=n^{\prime}+n-1$ and $k=n-1$, we have%
    \begin{align}
      \binom{n^{\prime}+n-1}{n-1} &  =\sum_{i=0}^{n-1}\binom{n}{i}\binom{n^{\prime}-1}{n-1-i},~\text{or}\\
      \binom{n^{\prime}+n-1}{n^{\prime}} &  =\sum_{i=0}^{n-1}\binom{n}{i}\binom{n^{\prime}-1}{i+n^{\prime}-n},~\text{then}\\
      1 &  =\sum_{i=0}^{n-1}\frac{\binom{n}{i}\binom{n^{\prime}-1}{i+n^{\prime}-n}}{\binom{n^{\prime}+n-1}{n^{\prime}}}%
    \end{align}

  \item For finding the expectation, consider first the identity%
    \begin{align}
      (x+y)^{n}(1+y)^{m} &  =\sum_{i=0}^{n}\binom{n}{i}x^{i}y^{n-i}\sum_{j=0}%
                           ^{m}\binom{m}{j}y^{j}\\
                         &  =\sum_{i=0}^{n}\sum_{j=0}^{m}\binom{n}{i}\binom{m}{j}x^{i}y^{n-i+j}%
                           \label{EQMatchCoefxy}%
    \end{align}
    Diffrentiating both sides w.r.t. $x$ then setting $x=1$ gives%
    \begin{align}
      n(1+y)^{m+n-1}=\sum_{i=0}^{n}\sum_{j=0}^{m}i\binom{n}{i}\binom{m}{j}y^{n-i+j}%
    \end{align}
    The coefficient of any power of $y$ in both sides should agree; hence, matching the coefficient of $y^{n}$ in both sides and
    setting $m=n-1$gives%
    \begin{align}
      n\binom{2n-2}{n}=\sum_{i=0}^{n-1}i\binom{n}{i}\binom{n-1}{i}%
    \end{align}
    Then%
    \begin{align}
      \operatorname*{E}a_{b} &  =\sum_{k=1}^{n-1}\frac{k\binom{n-1}{k}\binom{n}{k}%
                               }{\binom{2n-1}{n}}\\
                             &  =\frac{n\binom{2n-2}{n}}{\binom{2n-1}{n}}\\
                             &  =\frac{n\left(  n-1\right)  }{(2n-1)}%
    \end{align}

  \item Integrating both sides of (\ref{EQMatchCoefxy}) w.r.t. $x$ , setting $x=1$, and setting $m=n-1$ gives%
    \begin{align}
      \frac{1}{n+1}(1+y)^{2n}=\sum_{i=0}^{n}\sum_{j=0}^{n-1}\frac{1}{i+1}\binom
      {n}{i}\binom{n-1}{j}y^{n-i+j}.
    \end{align}
    Then matching the coefficeint of $y^{n}$ in both sides (by setting $i=j$ in the R.H.S.) gives%
    \begin{align}
      \frac{1}{n+1}\binom{2n}{n}=\sum_{i=0}^{n-1}\frac{1}{i+1}\binom{n}{i}%
      \binom{n-1}{i}.
    \end{align}
    Then%
    \begin{align}
      \operatorname*{E}\frac{1}{1+a_{b}} &  =\frac{1}{\binom{2n-1}{n}}\sum
                                           _{i=0}^{n-1}\frac{1}{i+1}\binom{n}{i}\binom{n-1}{i}\\
                                         &  =\frac{\frac{1}{n+1}\binom{2n}{n}}{\binom{2n-1}{n}}\\
                                         &  =\frac{2}{n+1},
    \end{align}
  \end{enumerate}
  which completes the proof.
\end{proof}

\begin{corollary}[0.632- or 0.5-bootstrap?]\label{COR632or5}
  The BS is supported on half of the observations; i.e., on average half of the observations appear in a BS replication, if we consider
  sampling with replacement \underline{without ordering}.
\end{corollary}
\begin{proof}
  The statement of this corollary was provided as \cite[Lemma 2][]{Yousef2019AUCSmoothness-arxiv} and was proved as follows: that an
  observation does not appear in a BS is equivalent to sampling with replacement and without ordering the $n$ observations from all $n$
  observations except that one. Then the probability to appear in this BS is%
  \begin{align}
    1-\Pr\left[  I_{i}^{b}=1\right]   &  =1-\frac{\binom{\left(  n-1\right)+n-1}{n}}{\binom{2n-1}{n}}\\
                                      &  =\frac{n}{\left(  2n-1\right)  }\cong\frac{1}{2}.
  \end{align}
  However, this result is immediate from the above Lemma, where the statement of the corollary could be posed as $a=\sum_{i}I_{i}$;
  then $\operatorname*{E}%
  a=\operatorname*{E}\sum_{i}I_{i}=\sum_{i}\operatorname*{E}I_{i}%
  =n\operatorname*{E}I_{1}~=n\Pr\left[ I_{1}=1\right] $ (from symmetry of $I$s). Then, $\Pr\left[ I_{1}=1\right]
  =\operatorname*{E}\left[ a_{b}/n\right] =\frac{n-1}{\left( 2n-1\right) }$; or $\Pr\left[I_{1}=0\right] =\frac{n}{2n-1}$.
\end{proof}

\begin{lemma}\label{CORExpWb}
  The r.v. $w_{b}=\frac{nI_{i}^{b}}{a_{b}}$ has a mean of $\frac{2n-2}{2n-1}$.
\end{lemma}
\begin{proof}
  If we denote $a_{b}$ by $a_{b}^{n}$ (to resemble sampling from $n$ observations), then we can say that
  $a_{b}^{n-1}=I_{i}^{b}+a_{b}^{n}$, in which case $I_{i}^{b}$ and $a_{b}^{n-1}$ are independent. Here, $a_{b}^{n-1}$ represents
  sampling from the remaining $n-1$ observations having the i$^{\text{th}}$ observation left-out from the BS, i.e., $I_{i}^{b}=1.$
  Then,%
  \begin{align}
    w_{b}=\left\{%
    \begin{array}[c]{cc}
      0 & \text{if }I_{i}^{b}=0\\
      \frac{n}{1+a_{b}^{n-1}} & \text{if }I_{i}^{b}=1
    \end{array}\right.,
  \end{align}
  where $0\leq a_{b}^{n-1}\leq n-2$ with pmf given by replacing $n~$by $n-1~$and $m$ by $n-1$ in (\ref{EQgenPMF}). Then,
  \begin{align}
    \operatorname*{E}w_{b}  & =\Pr[I_{i}^{b}=1]\operatorname*{E}\left[
                              w_{b}|I_{i}^{b}=1\right]  \\
                            & =\Pr[I_{i}^{b}=1]~n~\sum_{i=0}^{n-2}\frac{1}{i+1}\frac{\binom{n-1}{i}\binom{n-1}{i+1}}{\binom{2n-2}{n}}.
  \end{align}
  Integrating both sides of (\ref{EQMatchCoefxy}) w.r.t. $x$ gives
  \begin{align}
    \frac{1}{n+1}(x+y)^{n+1}(1+y)^{m}=\left[  \sum_{i=0}^{n}\frac{1}{i+1}\binom
    {n}{i}x^{i+1}y^{n-i}+\frac{1}{n+1}y^{n+1}\right]  \sum_{j=0}^{m}\binom{m}%
    {j}y^{j},
  \end{align}
  where the term $\frac{1}{n+1}y^{n+1}$ is the constant of integration. Setting $x=1$, and setting $m=n$ gives%
  \begin{align}
    \frac{1}{n+1}(1+y)^{2n+1}=\sum_{i=0}^{n}\sum_{j=0}^{n}\frac{1}{i+1}\binom
    {n}{i}\binom{n}{j}y^{n-i+j}+\frac{1}{n+1}\sum_{j=0}^{n}\binom{n}{j}y^{n+1+j}.
  \end{align}
  Then matching the coefficeint of $y^{n+1}$ in both sides (by setting $j=i+1$ in the double summation and $j=0$ in the single
  summation) gives%
  \begin{align}
    \frac{1}{n+1}\binom{2n+1}{n+1} &  =\sum_{i=0}^{n-1}\frac{1}{i+1}\binom{n}%
                                     {i}\binom{n}{i+1}+\frac{1}{n+1}\\
    \sum_{i=0}^{n-1}\frac{1}{i+1}\binom{n}{i}\binom{n}{i+1} &  =\frac{1}%
                                                              {n+1}\left[  \binom{2n+1}{n+1}-1\right].
  \end{align}
  So,%
  \begin{align}
    \sum_{i=0}^{n-2}\frac{1}{i+1}\binom{n-1}{i}\binom{n-1}{i+1} &  =\frac{1}{n}\left[  \binom{2n-1}{n}-1\right].
  \end{align}
  Then,%
  \begin{align}
    \operatorname*{E}w_{b} &  =\frac{n-1}{\left(  2n-1\right)  }n\frac{1}{n}\frac{\left[  \binom{2n-1}{n}-1\right]  }{\binom{2n-2}{n}}\\
                           &  =\frac{n!\left(  n-1\right)  !}{\left(  2n-1\right)  !}\left[  \binom{2n-1}{n}-1\right]  \\
                           &  =\frac{1}{\binom{2n-1}{n}}\left[  \binom{2n-1}{n}-1\right]  \\
                           &  =1-\Pr\left[  a_{b}=0\right]  \\
                           &  =\Pr[a_{b}\neq0].
  \end{align}
  The same result could have been obtained by observing that%
  \begin{align}
    \operatorname*{E}w_{b}  &  =\frac{1}{n}\sum_{i}\operatorname*{E}w_{b}=\frac{1}{n}\sum_{i}\operatorname*{E}\frac{nI_{i}^{b}}{a_{b}}\\
                            &  =\operatorname*{E}\sum_{i}\frac{I_{i}^{b}}{a_{b}}\\
                            &  =\operatorname*{E}\frac{a_{b}}{a_{b}}\\
                            &  =\Pr\left[  a_{b}\neq0\right].
  \end{align}
  The last equation follows from that $\frac{a_{b}}{a_{b}}=Ber(\Pr[a_b\neq0])$.
\end{proof}

\begin{proof}[\textbf{Proof of Equivalence of Eq.~\eqref{EQCVK} and~\eqref{EQCVKpart}}]
  \begin{subequations}
    \begin{align}
      \widehat{Err}^{\left(CVK*\right)} & =\frac{1}{K} \sum_{k} \left[\frac{1}{n_{K}} \sum_{i \in \mathcal{K}^{-1}\left( k\right) } Q\left( x_{i},\mathbf{X}_{\left(\left\{k\right\}\right)}\right) \right]\\
                                        & =\frac{1}{Kn_{K}}\sum_{k}\sum_{i \in \mathcal{K}^{-1}\left(  k\right)}Q\left(  x_{i},\mathbf{X}_{\left(  \left\{  k\right\}  \right)  }\right)\\
                                        & =\frac{1}{n} \sum_{i} Q\left(  x_{i}, \mathbf{X}_{\left(  \{\mathcal{K}\left(  i\right)  \}\right)  }\right)\\
                                        & = \widehat{Err}^{\left(CVK\right)},
    \end{align}
  \end{subequations}
  where, in the last step, $ \sum_{k}\sum_{i \in \mathcal{K}^{-1}\left( k\right)}$ and $\mathbf{X}_{\left( \left\{ k\right\} \right) }$
  are replaced by $\sum_{i}$ and $\mathbf{X}_{\left( \{\mathcal{K}\left( i\right) \}\right) }$ respectively.
\end{proof}

\begin{proof}[\textbf{Proof of Equivalence of Eq.~\eqref{eq:CVK-AUC} and~\eqref{eq:CVK*-AUC}}]
  \begin{subequations}
    \begin{align}
      \widehat{AUC}^{\left(CVK*\right)}& =\frac{1}{K_{1}K_{2}}\sum_{k_{1}=1}^{K_{1}}\sum_{k_{2}=1}^{K_{2}}\left[\frac{1}{n_{1K}n_{2K}}\sum_{i\in\mathcal{K}_{1}^{-1}(k_{1})}\sum_{j\in\mathcal{K}_{2}^{-1}(k_{2})}\psi\left(h\left(x_{i}\right), h\left(y_{j}\right)\right)\right], && h = h_{\mathbf{X1}_{\left(\left\{k_{1}\right\}\right)}\mathbf{X2}_{\left(\left\{k_{2}\right\}\right)}}\\
                                       &  =\frac{1}{K_{1}K_{2}n_{1K}n_{2K}} \sum_{k_2}\sum_{j\in\mathcal{K}_{2}^{-1}(k_{2})} \sum_{k_1}\sum_{i\in\mathcal{K}_{1}^{-1}(k_{1})} \psi\left(  h\left(x_{i}\right)  ,h\left(  y_{j}\right)\right),&& h = h_{\mathbf{X1}_{\left(\left\{k_{1}\right\}\right)}\mathbf{X2}_{\left(\left\{k_{2}\right\}\right)}}\\
                                       & =\frac{1}{n_{1}n_{2}} \sum_{j=1}^{n_{2}}\sum_{i=1}^{n_{1}}\psi\left(  h\left(  x_{i}\right)  , h\left(  y_{j}\right)  \right),&& h = h_{\mathbf{X1}_{\left(  \left\{  \mathcal{K}_{1}(i)\right\}  \right)  }\mathbf{X2}_{\left(\left\{  \mathcal{K}_{2}(j)\right\}  \right)  }}\\
                                       & = \widehat{AUC}^{\left(  CVK\right)  },
    \end{align}
  \end{subequations}
  where $\sum_{k_1} \sum_{i\in\mathcal{K}_{1}^{-1}(k_{1})}$, $\sum_{k_2} \sum_{j\in\mathcal{K}_{2}^{-1}(k_{2})}$,
  $\mathbf{X1}_{\left(\left\{k_{1}\right\}\right)}$, and $\mathbf{X2}_{\left(\left\{k_{2}\right\}\right)}$ are replaced by
  $\sum_{i=1}^{n_{1}}$, $\sum_{j=1}^{n_{2}}$, $\mathbf{X1}_{\left(\left\{\mathcal{K}_{1}(i)\right\}\right)}$, and
  $\mathbf{X2}_{\left(\left\{\mathcal{K}_{2}(j)\right\}\right)}$ respectively.
\end{proof}

\begin{proof}[\textbf{Proof of Equivalence of Eq.~\eqref{EqCVKRa} and~\eqref{EqCVKRpart}}]
  \begin{subequations}
    \begin{align}
      \widehat{Err}^{\left(CVKR*\right)} &  =\frac{1}{M}\sum_{m=1}^{M} \left[\frac{1}{K}\sum_{k}  \frac{1}{n_{K}} \sum_{i\in\mathcal{K}_{m}^{-1}\left(k\right)} Q\left(x_{i},\mathbf{X}_{\left(\left\{k\right\} ,m\right)}\right)\right]\\
                                         &  =\frac{1}{M}\sum_{m=1}^{M}\left[  \frac{1}{n}\sum_{i=1}^{n}Q\left(x_{i},\mathbf{X}_{\left(  \left\{  \mathcal{K}_{m}\left(  i\right)  \right\},m\right)  }\right)  \right]\\
                                         & = \frac{1}{n}\sum_{i=1}^{n}\left[\frac{1}{M}  \sum_{m}Q\left(  x_{i},\mathbf{X}_{\left(  \left\{  \mathcal{K}_{m}\left(  i\right)  \right\}  ,m\right)  }\right)\right]\\
                                         & = \widehat{Err}^{\left(CVKR\right)}.
    \end{align}
  \end{subequations}

\end{proof}

\begin{proof}[\textbf{Proof of Equivalence of Eq.~\eqref{eq:CVKR-AUC} and~\eqref{eq:CVKR-AUCstar}}]
  \begin{subequations}
    \begin{align}
      \widehat{AUC}^{\left(  CVKR*\right)  }  & =\frac{1}{M}\sum_{m=1}^{M} \left[\frac{1}{K_{1}K_{2}}\sum_{k_{1}=1}^{K_{1}}\sum_{k_{2}=1}^{K_{2}} \frac{1}{n_{1K}n_{2K}}\sum_{i\in\mathcal{K}_{1m}^{-1}(k_{1})}\sum_{j\in\mathcal{K}_{2m}^{-1}(k_{2})}\psi\left(h\left(x_{i}\right), h\left(y_{j}\right)\right)\right],&&  h = h_{\mathbf{X1}_{\left(  \left\{  k_{1}\right\}  \right)  }\mathbf{X_2}_{\left(  \left\{k_{2}\right\}  \right)  }}\\
                                              &  =\frac{1}{M}\sum_{m=1}^{M}\left[  \frac{1}{n_{1}n_{2}}\sum_{j=1}^{n_{2}}\sum_{i=1}^{n_{1}}\psi\left( h\left(x_{i}\right), h\left(y_{j}\right)\right)\right],&&  h = h_{\mathbf{X1}_{\left(\left\{  \mathcal{K}_{1m}(i)\right\}  \right)  }\mathbf{X_2}_{\left(  \left\{\mathcal{K}_{2m}(j)\right\}  \right)  }}\\
                                              & = \frac{1}{n_{1}n_{2}}\sum_{j=1}^{n_{2}}\sum_{i=1}^{n_{1}}\left[\frac{1}{M} \sum_{m}\psi\left(h\left(x_{i}\right), h\left(  y_{j}\right)  \right)\right],&&  h= h_{\mathbf{X1}_{\left(  \left\{  \mathcal{K}_{1m}(i)\right\}\right)  }\mathbf{X2}_{\left(  \left\{  \mathcal{K}_{2m}(j)\right\}  \right)}}\\
                                              & = \widehat{AUC}^{\left(  CVKR\right)  }.
    \end{align}
  \end{subequations}
\end{proof}

\begin{proof}[\textbf{Proof of Non Equivalence of Eq.~\eqref{EqCVKM} and~\eqref{EqCVKM*a}}]
  \begin{subequations}
    \begin{align}
      \widehat{Err}^{\left(  CVKM\ast\right)  }  &  =\frac{1}{M}\sum_{m=1}^{M}\left[  \frac{1}{n_{K}}\sum_{i\in\mathcal{K}_{m}^{-1}\left(  1\right)}Q\left(  x_{i},\mathbf{X}_{\left(\{1\}, m\right)}\right)\right]\\
                                                 & =\frac{1}{M}\sum_{m=1}^{M}\left[\frac{1}{n_K} \sum_{i}I_{i}^{m}Q\left(x_{i},\mathbf{X}_{\left(\{1\}, m\right)}\right)\right]\\
                                                 & =\frac{1}{n}\sum_{i}\left[  \left.  \left(  \frac{n}{n_{K}}\frac{1}{M}\sum_{m}I_{i}^{m}\right)  \sum_{m=1}^{M}I_{i}^{m}Q\left(  x_{i},\mathbf{X}_{\left(\{1\}, m\right)}\right)  \right/  \sum_{m}I_{i}^{m}\right]\\
                                                 & \neq \widehat{Err}^{\left(  CVKM\right)  }.
    \end{align}
  \end{subequations}
  From the S.L.L.N, we know that $\frac {1}{M}\sum_{m}I_{i}^{m}\overset{a.s.}{\rightarrow}\operatorname*{E}\left[ I_{i}^{m}\right]
  =\frac{n_{K}}{n}$ as $M\rightarrow\infty$. Hence, both variants are asymptotically equivalent.
\end{proof}

\begin{proof}[\textbf{Proof of Non Equivalence of Eq.~\eqref{eq:CVKM-AUC} and~\eqref{eq:AUC-CVKM*}}]
  \begin{subequations}
    \begin{align}
      \widehat{AUC}^{\left(  CVKM\ast\right)  }  &  =\frac{1}{M}\sum_{m=1}^{M}\left[  \frac{1}{n_{1K}n_{2K}}\sum_{j\in\mathcal{K}_{2m}^{-1}\left(1\right)  }\sum_{i\in\mathcal{K}_{1m}^{-1}\left(  1\right)  }\psi\left(h\left(x_{i}\right), h\left(y_{j}\right)\right)\right], && h = h_{\mathbf{X1}_{\left(\{1\},m\right)}\mathbf{X2}_{\left(\{1\},m\right)}}\\
                                                &  =\frac{1}{M}\sum_{m=1}^{M}\left[\frac{1}{n_{1K}n_{2K}} \sum_{j}\sum_{i}I_{i}^{m}I_{j}^{m}\psi\left(h\left(x_{i}\right), h\left(y_{j}\right)\right)\right], && h = h_{\mathbf{X1}_{\left(\{1\},m\right)  }\mathbf{X2}_{\left(\{1\},m\right)}}\\
                                                &  =\frac{1}{n_{1}n_{2}}\sum_{j}\sum_{i}\left[  \left.  \frac{n_{1}n_{2}\sum_{m}I_{i}^{m}I_{j}^{m}}{Mn_{1K}n_{2K}}\sum_{m}I_{j}^{m}I_{i}^{m}\psi\left(h\left(x_{i}\right), h\left(y_{j}\right)\right)\right/\sum_{m}I_{i}^{m}I_{j}^{m}\right], && h = h_{\mathbf{X1}_{\left(\{1\},m\right)}\mathbf{X2}_{\left(\{1\}  ,m\right)  }}\\
                                                & \neq \widehat{AUC}^{\left(  CVKM\right)  }.
    \end{align}
  \end{subequations}
  The two variants are not equal for finite $M$. However, asymptotically as $M\rightarrow\infty$, we know that from the Strong Law for
  Large Numbers (SLLN) $\frac{1}{M}\sum_{m}I_{i}^{m}I_{j}^{m}\overset{a.s.}{\rightarrow}\operatorname*{E}\left[
    I_{i}^{m}I_{j}^{m}\right] =\frac{n_{1K}n_{2K}}{n_{1}n_{2}}$, as $M\rightarrow\infty$. Hence, the two variants are equivalent as
  $M\rightarrow\infty$.
\end{proof}

\begin{proof}[\textbf{Proof of Non Equivalence of \eqref{EqLOOerr} and~\eqref{EqStarErr}}]
  \begin{subequations}
    \begin{align}
      \widehat{Err}^{\left( \ast\right) } & =\frac{1}{B}\sum_{b}\left[ \sum _{i}I_{i}^{b}Q\left( x_{i},\mathbf{X}^{\ast b}\right) \left/ \sum_{i} I_{i}^{b}\right.  \right]\\
                                         & =\frac{1}{n}\sum_{i}\sum_{b}\frac{\left(\frac{nI_{i}^{b}}{a_{b}}\right) Q\left( x_{i},\mathbf{X}^{\ast b}\right) }{B},\label{eq:EqStarErr2}
    \end{align}
  \end{subequations}
  where $a_{b}=\sum_{i}I_{i}^{b}$. From Corollary~\ref{COR632or5}, $I_{i}^{b}$ is a $Ber(\frac{n-1}{2n-1})$ r.v. $\forall~i$. Similar to the
  CVKM, it is obvious that the two variants are not equal for finite bootstraps $B$; and by comparing equations (\ref{EqLOOerr}) and
  (\ref{EqStarErr}), they can be seen as summation with different weights. To analyze them we denote $\frac{nI_{i}^{b}}{a_{b}}$ by
  $w_{b}$; then direct application of Lemma~\ref{LEMIiXiConvergence} gives
  \begin{align}
    \widehat{Err}^{\left(  \ast\right)  }\overset{a.s.}{\rightarrow}\frac{1} {n}\sum_{i}\operatorname*{E}\left[  w_{b}\right]  err^{i}=\operatorname*{E} \left[  w_{b}\right]  \left(  \frac{1}{n}\sum_{i}err^{i}\right)  .
  \end{align}
  This means that the ratio between the two estimators (as $B\rightarrow\infty$): (1) does not depend on the classifier rather it
  depends on the sampling mechanism (exactly as was the case for the CVKM). (2) The two estimators are not equal even asymptotically
  with $B$. From Lemma~\ref{CORExpWb}%
  \begin{align}
    \widehat{Err}^{\left(  \ast\right)  }\overset{a.s.}{\rightarrow}\left(\frac{2n-2}{2n-1}\right)  \frac{1}{n}\sum_{i}err^{i}~\text{as }B\rightarrow\infty.
  \end{align}
  Suppose $\frac{1}{B}\sum_{b}Q(x_{i},\mathbf{X}^{\ast b})\overset{a.s.}{\rightarrow}err^{i}$ as $B\rightarrow\infty$, which means that
  with increasing the number of bootstraps infinitely the average converges to some function of $x_{i}$, named $err^{i}$, which is an
  estimate of the population error rate of the observation $x_i$. Then, by Lemma~\ref{LEMIiXiConvergence},
  $\sum_{b}\frac{I_{i}^{b}Q\left( x_{i},\mathbf{X}^{\ast b}\right) }{\sum_{b}I_{i}^{b}} \overset {a.s.}{\rightarrow}err^{i}$ as well,
  and hence%
  \begin{align}
    \widehat{Err}^{\left(  1\right)  }\overset{a.s.}{\rightarrow}\frac{1}{n}\sum_{i}err^{i},\ \text{as}\ B\rightarrow\infty.
  \end{align}
\end{proof}

\begin{proof}[\textbf{Proof of Non Equivalence of Eq. \eqref{eq:BS-AUC} and~\eqref{eq:AUCstar}}]
  \begin{subequations}
    \begin{align}
      \widehat{AUC}^{\left( \ast\right) } & =\frac{1}{B}\sum_{b}\left[ \sum _{j}\sum_{i}I_{i}^{b}I_{j}^{b}\psi\left( h_{\mathbf{X}^{\ast b}}( x_{i}) ,h_{\mathbf{X}^{\ast b}}( y_{j}) \right) \left/ \sum_{j}\sum_{i}I_{i}^{b}I_{j}^{b}\right.\right]\\
                                         & =\frac{1}{n_{1}n_{2}}\sum_{j}\sum_{i}\sum_{b}\frac{\left( \frac{n_{1} n_{2}I_{i}^{b}I_{j}^{b}}{A_{b}}\right) \psi\left( h_{\mathbf{X}^{\ast b}}( x_{i}) ,h_{\mathbf{X}^{\ast b}}( y_{j}) \right) }{B},
    \end{align}%
  \end{subequations}
  where $A_{b}=\sum_{j}\sum_{i}I_{i}^{b}I_{j}^{b}=\sum_{i}I_{i}^{b}\sum_{j}I_{j}^{b}$. Arranging gives%
  \begin{align}
    \frac{n_{1}n_{2}I_{i}^{b}I_{j}^{b}}{A_{b}}  &  =\frac{n_{1}I_{i}^{b}}{\sum_{i}I_{i}^{b}} \cdot \frac{n_{2}I_{j}^{b}}{\sum_{j}I_{j}^{b}} = w_{b}w_{b}^{\prime},
  \end{align}%
  where $w_{b}$ and $w_{b}^{\prime}$ are independent because of stratification. Then, again by Lemma~\ref{LEMIiXiConvergence}%
  \begin{align}
    \widehat{AUC}^{\left(  \ast\right)  }\overset{a.s.}{\rightarrow}\frac{1}{n_{1}n_{2}}\sum_{j}\sum_{i}\operatorname*{E}\left[  w_{b}w_{b}^{\prime}\right]  AUC^{i,j}  &  =\operatorname*{E}\left[  w_{b}\right]\operatorname*{E}\left[  w_{b}^{\prime}\right]  \left(  \frac{1}{n_{1}n_{2}}\sum_{j}\sum_{i}AUC^{i,j}\right) \\
                                                                                                                                                               &  =\left(  \frac{2n_{1}-2}{2n_{1}-1}\right)  \left(  \frac{2n_{2}-2}{2n_{2}-1}\right)  \left(  \frac{1}{n_{1}n_{2}}\sum_{j}\sum_{i}AUC^{i,j}\right).
  \end{align}%
  Suppose $\frac{1}{B}\sum_{b}\psi\left(h(x_{i}), h(y_{j})\right)\overset{a.s.}{\rightarrow}AUC^{i,j}$ as $B\rightarrow\infty$, $h =
  h_{\mathbf{X}^{\ast b}}$, which means that with increasing the number of bootstraps infinitely the average converges to some function
  of $x_{i}$ and $y_{j}$, named $AUC^{i,j}$, which is an estimate for the population MannWhitney statistic of the observations $x_i$
  and $y_j$. Since $I_{j}^{b}I_{i}^{b}$ is Bernoulli r.v., by Lemma~\ref{LEMIiXiConvergence},
  $\sum_{b}\frac{I_{j}^{b}I_{i}^{b}\psi\left(h\left(x_{i}\right), h\left(y_{j}\right)\right)}{\sum_{b}I_{j}^{b}I_{i}^{b}}$
  $\overset{a.s.}{\rightarrow}AUC^{i,j}$, $h = h_{\mathbf{X}^{\ast b}}$, and hence%
  \begin{align}
    \widehat{AUC}^{\left(  1,1\right)  }\overset{a.s.}{\rightarrow}\frac{1}{n_{1}n_{2}}\sum_{j}\sum_{i}AUC^{i,j}\ \text{as}\ B\rightarrow\infty.
  \end{align}%
\end{proof}

\begin{proof}[\textbf{Proof of Equivalence between CVKM and CVKR}]
  This conclusion is immediate by comparing these two equations, if $\left. \sum _{m}Q\left( x_{i},\mathbf{X}_{\left( \left\{
            \mathcal{K}_{m}\left( i\right) \right\},m\right) }\right) \right/ M$ converges a.s.\ as $M\rightarrow\infty$ then by Lemma
  \ref{LEMIiXiConvergence}, $\left. \sum _{m}I_{i}^{m}Q\left( x_{i},\mathbf{X}_{\left( \left\{ 1\right\} ,m\right) }\right) \right/
  \sum_{m}I_{i}^{m}$ converges, as well, a.s.\ to the same value. Hence, these two versions of CV produces the same estimate as
  $M\rightarrow\infty$. However, $I_{i}^{m}$ is $Ber(n_{K}/n)$, which of course delays the convergence. For the AUC estimators, we
  observe that each of $I_{i}^{m}I_{j}^{m}$ in~\eqref{eq:11} and $I_{i}^{b}I_{j}^{b}$ in~\eqref{eq:12} is Bernoulli r.v. A conclusion
  identical to that following~\eqref{eq:6} for the error rate estimators is immediate here for the AUC estimators. This is obvious
  since that conclusion is inherent in the resampling mechanism rather than the estimated measure.
\end{proof}


  \small
  \let\oldbibliography\thebibliography
  \renewcommand{\thebibliography}[1]{%
    \oldbibliography{#1}%
    \setlength{\itemsep}{0pt}%
  }
  \putbib[booksIhave,publications]
\end{bibunit}


\begin{thebibliography}{12}
\expandafter\ifx\csname natexlab\endcsname\relax\def\natexlab#1{#1}\fi
\providecommand{\url}[1]{\texttt{#1}}
\providecommand{\href}[2]{#2}
\providecommand{\path}[1]{#1}
\providecommand{\DOIprefix}{doi:}
\providecommand{\ArXivprefix}{arXiv:}
\providecommand{\URLprefix}{URL: }
\providecommand{\Pubmedprefix}{pmid:}
\providecommand{\doi}[1]{\href{http://dx.doi.org/#1}{\path{#1}}}
\providecommand{\Pubmed}[1]{\href{pmid:#1}{\path{#1}}}
\providecommand{\bibinfo}[2]{#2}
\ifx\xfnm\relax \def\xfnm[#1]{\unskip,\space#1}\fi
\bibitem[{Efron(1983)}]{Efron1983EstimatingTheError}
\bibinfo{author}{Efron, B.}, \bibinfo{year}{1983}.
\newblock \bibinfo{title}{{Estimating the Error Rate of a Prediction Rule:
  Improvement on Cross-Validation}}.
\newblock \bibinfo{journal}{Journal of the American Statistical Association}
  \bibinfo{volume}{78}, \bibinfo{pages}{316--331}.
\bibitem[{Efron(1992)}]{Efron1992JackknifeAfter}
\bibinfo{author}{Efron, B.}, \bibinfo{year}{1992}.
\newblock \bibinfo{title}{{Jackknife-After-Bootstrap Standard Errors and
  Influence Functions}}.
\newblock \bibinfo{journal}{Journal of the Royal Statistical Society. Series B
  (Methodological)} \bibinfo{volume}{54}, \bibinfo{pages}{83--127}.
\bibitem[{Efron and Tibshirani(1993)}]{Efron1993AnIntroduction}
\bibinfo{author}{Efron, B.}, \bibinfo{author}{Tibshirani, R.},
  \bibinfo{year}{1993}.
\newblock \bibinfo{title}{{An introduction to the bootstrap}}.
\newblock \bibinfo{publisher}{Chapman and Hall}, \bibinfo{address}{New York}.
\bibitem[{Efron and Tibshirani(1997)}]{Efron1997ImprovementsOnCross}
\bibinfo{author}{Efron, B.}, \bibinfo{author}{Tibshirani, R.},
  \bibinfo{year}{1997}.
\newblock \bibinfo{title}{{Improvements on Cross-Validation: the $.632+$
  Bootstrap Method}}.
\newblock \bibinfo{journal}{Journal of the American Statistical Association}
  \bibinfo{volume}{92}, \bibinfo{pages}{548--560}.
\bibitem[{H\'{a}jek et~al.(1999)H\'{a}jek, \v{S}id\'{a}k and
  Sen}]{Hajek1999TheoryOfRank}
\bibinfo{author}{H\'{a}jek, J.}, \bibinfo{author}{\v{S}id\'{a}k, Z.},
  \bibinfo{author}{Sen, P.K.}, \bibinfo{year}{1999}.
\newblock \bibinfo{title}{{Theory of rank tests}}.
\newblock \bibinfo{edition}{2nd} ed., \bibinfo{publisher}{Academic Press},
  \bibinfo{address}{San Diego, Calif.}
\bibitem[{Randles and Wolfe(1979)}]{Randles1979IntroductionTo}
\bibinfo{author}{Randles, R.H.}, \bibinfo{author}{Wolfe, D.A.},
  \bibinfo{year}{1979}.
\newblock \bibinfo{title}{{Introduction to the theory of nonparametric
  statistics}}.
\newblock \bibinfo{publisher}{Wiley}, \bibinfo{address}{New York}.
\bibitem[{Sahiner et~al.(2008)Sahiner, Chan and
  Hadjiiski}]{Sahiner2008ClassifierPerformance}
\bibinfo{author}{Sahiner, B.}, \bibinfo{author}{Chan, H.P.},
  \bibinfo{author}{Hadjiiski, L.}, \bibinfo{year}{2008}.
\newblock \bibinfo{title}{{Classifier Performance Prediction for Computer-Aided
  Diagnosis Using a Limited dataset}}.
\newblock \bibinfo{journal}{Medical Physics} \bibinfo{volume}{35},
  \bibinfo{pages}{1559}.
\bibitem[{Sahiner et~al.(2001)Sahiner, Chan, Petrick, Hadjiiski, Paquerault and
  Gurcan}]{Sahiner2001ResamplingSchemes}
\bibinfo{author}{Sahiner, B.}, \bibinfo{author}{Chan, H.P.},
  \bibinfo{author}{Petrick, N.}, \bibinfo{author}{Hadjiiski, L.},
  \bibinfo{author}{Paquerault, S.}, \bibinfo{author}{Gurcan, M.N.},
  \bibinfo{year}{2001}.
\newblock \bibinfo{title}{{Resampling Schemes for Estimating the Accuracy of a
  Classifier Designed With a Limited Data Set}}.
\newblock \bibinfo{journal}{Medical Image Perception Conference IX, Airlie
  Conference Center, Warrenton VA, 20-23} .
\bibitem[{Yousef(2022)}]{Yousef2022MachineLearning}
\bibinfo{author}{Yousef, W.A.}, \bibinfo{year}{2022}.
\newblock \bibinfo{title}{Machine learning: Assessment}, in:
  \bibinfo{editor}{Traoré, I.}, \bibinfo{editor}{Woungang, I.},
  \bibinfo{editor}{Saad, S.} (Eds.), \bibinfo{booktitle}{Artificial
  Intelligence For Cyber-Physical Systems Hardening}.
  \bibinfo{publisher}{Springer}.
\bibitem[{Yousef et~al.(2004)Yousef, Wagner and Loew}]{Yousef2004ComparisonOf}
\bibinfo{author}{Yousef, W.A.}, \bibinfo{author}{Wagner, R.F.},
  \bibinfo{author}{Loew, M.H.}, \bibinfo{year}{2004}.
\newblock \bibinfo{title}{{Comparison of Non-Parametric Methods for Assessing
  Classifier Performance in Terms of ROC Parameters}}, in:
  \bibinfo{booktitle}{Applied Imagery Pattern Recognition Workshop, 2004.
  Proceedings. 33rd; IEEE Computer Society}, pp. \bibinfo{pages}{190--195}.
\bibitem[{Yousef et~al.(2005)Yousef, Wagner and Loew}]{Yousef2005EstimatingThe}
\bibinfo{author}{Yousef, W.A.}, \bibinfo{author}{Wagner, R.F.},
  \bibinfo{author}{Loew, M.H.}, \bibinfo{year}{2005}.
\newblock \bibinfo{title}{{Estimating the Uncertainty in the Estimated Mean
  Area Under the ROC Curve of a Classifier}}.
\newblock \bibinfo{journal}{Pattern Recognition Letters} \bibinfo{volume}{26},
  \bibinfo{pages}{2600--2610}.
\bibitem[{Zhang(1995)}]{Zhang1995AssessingPrediction}
\bibinfo{author}{Zhang, P.}, \bibinfo{year}{1995}.
\newblock \bibinfo{title}{{Assessing Prediction Error in Nonparametric
  Regression}}.
\newblock \bibinfo{journal}{Scandinavian Journal Of Statistics}
  \bibinfo{volume}{22}, \bibinfo{pages}{83--94}.

\end{thebibliography}


\begin{thebibliography}{12}
\expandafter\ifx\csname natexlab\endcsname\relax\def\natexlab#1{#1}\fi
\providecommand{\url}[1]{\texttt{#1}}
\providecommand{\href}[2]{#2}
\providecommand{\path}[1]{#1}
\providecommand{\DOIprefix}{doi:}
\providecommand{\ArXivprefix}{arXiv:}
\providecommand{\URLprefix}{URL: }
\providecommand{\Pubmedprefix}{pmid:}
\providecommand{\doi}[1]{\href{http://dx.doi.org/#1}{\path{#1}}}
\providecommand{\Pubmed}[1]{\href{pmid:#1}{\path{#1}}}
\providecommand{\bibinfo}[2]{#2}
\ifx\xfnm\relax \def\xfnm[#1]{\unskip,\space#1}\fi
\bibitem[{Bengio and Grandvalet(2004)}]{Bengio2004NoUnbiasedEstKCV}
\bibinfo{author}{Bengio, Y.}, \bibinfo{author}{Grandvalet, Y.},
  \bibinfo{year}{2004}.
\newblock \bibinfo{title}{{No Unbiased Estimator of the Variance of K-Fold
  Cross-Validation}}.
\newblock \bibinfo{journal}{J. Mach. Learn. Res.} \bibinfo{volume}{5},
  \bibinfo{pages}{1089--1105}.
\bibitem[{Chen et~al.(2012a)Chen, Gallas and Yousef}]{Chen2012ClassVar}
\bibinfo{author}{Chen, W.}, \bibinfo{author}{Gallas, B.D.},
  \bibinfo{author}{Yousef, W.A.}, \bibinfo{year}{2012}a.
\newblock \bibinfo{title}{{Classifier Variability: Accounting for Training and
  testing}}.
\newblock \bibinfo{journal}{Pattern Recognition} \bibinfo{volume}{45},
  \bibinfo{pages}{2661--2671}.
\bibitem[{Chen et~al.(2012b)Chen, Yousef, Gallas, Hsu, Lababidi, Tang,
  Pennello, Symmans and Pusztai}]{Chen2012UncertEst}
\bibinfo{author}{Chen, W.}, \bibinfo{author}{Yousef, W.A.},
  \bibinfo{author}{Gallas, B.D.}, \bibinfo{author}{Hsu, E.R.},
  \bibinfo{author}{Lababidi, S.}, \bibinfo{author}{Tang, R.},
  \bibinfo{author}{Pennello, G.A.}, \bibinfo{author}{Symmans, W.F.},
  \bibinfo{author}{Pusztai, L.}, \bibinfo{year}{2012}b.
\newblock \bibinfo{title}{{Uncertainty Estimation With a Finite Dataset in the
  Assessment of Classification models}}.
\newblock \bibinfo{journal}{Computational Statistics {\&} Data Analysis}
  \bibinfo{volume}{56}, \bibinfo{pages}{1016--1027}.
\bibitem[{Efron(1983)}]{Efron1983EstimatingTheError}
\bibinfo{author}{Efron, B.}, \bibinfo{year}{1983}.
\newblock \bibinfo{title}{{Estimating the Error Rate of a Prediction Rule:
  Improvement on Cross-Validation}}.
\newblock \bibinfo{journal}{Journal of the American Statistical Association}
  \bibinfo{volume}{78}, \bibinfo{pages}{316--331}.
\bibitem[{Efron and Tibshirani(1997)}]{Efron1997ImprovementsOnCross}
\bibinfo{author}{Efron, B.}, \bibinfo{author}{Tibshirani, R.},
  \bibinfo{year}{1997}.
\newblock \bibinfo{title}{{Improvements on Cross-Validation: the $.632+$
  Bootstrap Method}}.
\newblock \bibinfo{journal}{Journal of the American Statistical Association}
  \bibinfo{volume}{92}, \bibinfo{pages}{548--560}.
\bibitem[{Randles and Wolfe(1979)}]{Randles1979IntroductionTo}
\bibinfo{author}{Randles, R.H.}, \bibinfo{author}{Wolfe, D.A.},
  \bibinfo{year}{1979}.
\newblock \bibinfo{title}{{Introduction to the theory of nonparametric
  statistics}}.
\newblock \bibinfo{publisher}{Wiley}, \bibinfo{address}{New York}.
\bibitem[{Shi et~al.(2010)Shi, Campbell, Jones, Campagne, Wen, Walker, Su, Chu,
  Goodsaid, Pusztai, {Shaughnessy Jr.}, Oberthuer, Thomas, Paules, Fielden,
  Barlogie, Chen, Du, Fischer, Furlanello, Gallas, Ge, Megherbi, Symmans, Wang,
  Zhang, Bitter, Brors, Bushel, Bylesjo, Chen, Cheng, Chou, Davison, Delorenzi,
  Deng, Devanarayan, Dix, Dopazo, Dorff, Elloumi, Fan, Fan, Fan, Fang,
  Gonzaludo, Hess, Hong, Huan, Irizarry, Judson, Juraeva, Lababidi, Lambert,
  Li, Li, Li, Lin, Liu, Lobenhofer, Luo, Luo, McCall, Nikolsky, Pennello,
  Perkins, Philip, Popovici, Price, Qian, Scherer, Shi, Shi, Sung,
  Thierry-Mieg, Thierry-Mieg, Thodima, Trygg, Vishnuvajjala, Wang, Wu, Wu, Xie,
  Yousef, Zhang, Zhang, Zhong, Zhou, Zhu, Arasappan, Bao, Lucas, Berthold,
  Brennan, Buness, Catalano, Chang, Chen, Cheng, Cui, Czika, Demichelis, Deng,
  Dosymbekov, Eils, Feng, Fostel, Fulmer-Smentek, Fuscoe, Gatto, Ge, Goldstein,
  Guo, Halbert, Han, Harris, Hatzis, Herman, Huang, Jensen, Jiang, Johnson,
  Jurman, Kahlert, Khuder, Kohl, Li, Li, Li, Li, Liu, Liu, Liu, Meng, Madera,
  Martinez-Murillo, Medina, Meehan, Miclaus, Moffitt, Montaner, Mukherjee,
  Mulligan, Neville, Nikolskaya, Ning, Page, Parker, Parry, Peng, Peterson,
  Phan, Quanz, Ren, Riccadonna, Roter, Samuelson, Schumacher, Shambaugh, Shi,
  Shippy, Si, Smalter, Sotiriou, Soukup, Staedtler, Steiner, Stokes, Sun, Tan,
  Tang, Tezak, Thorn, Tsyganova, Turpaz, Vega, Visintainer, von Frese, Wang,
  Wang, Wang, Wang, Westermann, Willey, Woods, Wu, Xiao, Xu, Xu, Yang, Zeng,
  Zhang, Zhao, Puri, Scherf, Tong and Wolfinger}]{Shi2010MAQCII}
\bibinfo{author}{Shi, L.}, \bibinfo{author}{Campbell, G.},
  \bibinfo{author}{Jones, W.D.}, \bibinfo{author}{Campagne, F.},
  \bibinfo{author}{Wen, Z.}, \bibinfo{author}{Walker, S.J.},
  \bibinfo{author}{Su, Z.}, \bibinfo{author}{Chu, T.M.},
  \bibinfo{author}{Goodsaid, F.M.}, \bibinfo{author}{Pusztai, L.},
  \bibinfo{author}{{Shaughnessy Jr.}, J.D.}, \bibinfo{author}{Oberthuer, A.},
  \bibinfo{author}{Thomas, R.S.}, \bibinfo{author}{Paules, R.S.},
  \bibinfo{author}{Fielden, M.}, \bibinfo{author}{Barlogie, B.},
  \bibinfo{author}{Chen, W.}, \bibinfo{author}{Du, P.},
  \bibinfo{author}{Fischer, M.}, \bibinfo{author}{Furlanello, C.},
  \bibinfo{author}{Gallas, B.D.}, \bibinfo{author}{Ge, X.},
  \bibinfo{author}{Megherbi, D.B.}, \bibinfo{author}{Symmans, W.F.},
  \bibinfo{author}{Wang, M.D.}, \bibinfo{author}{Zhang, J.},
  \bibinfo{author}{Bitter, H.}, \bibinfo{author}{Brors, B.},
  \bibinfo{author}{Bushel, P.R.}, \bibinfo{author}{Bylesjo, M.},
  \bibinfo{author}{Chen, M.}, \bibinfo{author}{Cheng, J.},
  \bibinfo{author}{Chou, J.}, \bibinfo{author}{Davison, T.S.},
  \bibinfo{author}{Delorenzi, M.}, \bibinfo{author}{Deng, Y.},
  \bibinfo{author}{Devanarayan, V.}, \bibinfo{author}{Dix, D.J.},
  \bibinfo{author}{Dopazo, J.}, \bibinfo{author}{Dorff, K.C.},
  \bibinfo{author}{Elloumi, F.}, \bibinfo{author}{Fan, J.},
  \bibinfo{author}{Fan, S.}, \bibinfo{author}{Fan, X.}, \bibinfo{author}{Fang,
  H.}, \bibinfo{author}{Gonzaludo, N.}, \bibinfo{author}{Hess, K.R.},
  \bibinfo{author}{Hong, H.}, \bibinfo{author}{Huan, J.},
  \bibinfo{author}{Irizarry, R.A.}, \bibinfo{author}{Judson, R.},
  \bibinfo{author}{Juraeva, D.}, \bibinfo{author}{Lababidi, S.},
  \bibinfo{author}{Lambert, C.G.}, \bibinfo{author}{Li, L.},
  \bibinfo{author}{Li, Y.}, \bibinfo{author}{Li, Z.}, \bibinfo{author}{Lin,
  S.M.}, \bibinfo{author}{Liu, G.}, \bibinfo{author}{Lobenhofer, E.K.},
  \bibinfo{author}{Luo, J.}, \bibinfo{author}{Luo, W.},
  \bibinfo{author}{McCall, M.N.}, \bibinfo{author}{Nikolsky, Y.},
  \bibinfo{author}{Pennello, G.A.}, \bibinfo{author}{Perkins, R.G.},
  \bibinfo{author}{Philip, R.}, \bibinfo{author}{Popovici, V.},
  \bibinfo{author}{Price, N.D.}, \bibinfo{author}{Qian, F.},
  \bibinfo{author}{Scherer, A.}, \bibinfo{author}{Shi, T.},
  \bibinfo{author}{Shi, W.}, \bibinfo{author}{Sung, J.},
  \bibinfo{author}{Thierry-Mieg, D.}, \bibinfo{author}{Thierry-Mieg, J.},
  \bibinfo{author}{Thodima, V.}, \bibinfo{author}{Trygg, J.},
  \bibinfo{author}{Vishnuvajjala, L.}, \bibinfo{author}{Wang, S.J.},
  \bibinfo{author}{Wu, J.}, \bibinfo{author}{Wu, Y.}, \bibinfo{author}{Xie,
  Q.}, \bibinfo{author}{Yousef, W.A.}, \bibinfo{author}{Zhang, L.},
  \bibinfo{author}{Zhang, X.}, \bibinfo{author}{Zhong, S.},
  \bibinfo{author}{Zhou, Y.}, \bibinfo{author}{Zhu, S.},
  \bibinfo{author}{Arasappan, D.}, \bibinfo{author}{Bao, W.},
  \bibinfo{author}{Lucas, A.B.}, \bibinfo{author}{Berthold, F.},
  \bibinfo{author}{Brennan, R.J.}, \bibinfo{author}{Buness, A.},
  \bibinfo{author}{Catalano, J.G.}, \bibinfo{author}{Chang, C.},
  \bibinfo{author}{Chen, R.}, \bibinfo{author}{Cheng, Y.},
  \bibinfo{author}{Cui, J.}, \bibinfo{author}{Czika, W.},
  \bibinfo{author}{Demichelis, F.}, \bibinfo{author}{Deng, X.},
  \bibinfo{author}{Dosymbekov, D.}, \bibinfo{author}{Eils, R.},
  \bibinfo{author}{Feng, Y.}, \bibinfo{author}{Fostel, J.},
  \bibinfo{author}{Fulmer-Smentek, S.}, \bibinfo{author}{Fuscoe, J.C.},
  \bibinfo{author}{Gatto, L.}, \bibinfo{author}{Ge, W.},
  \bibinfo{author}{Goldstein, D.R.}, \bibinfo{author}{Guo, L.},
  \bibinfo{author}{Halbert, D.N.}, \bibinfo{author}{Han, J.},
  \bibinfo{author}{Harris, S.C.}, \bibinfo{author}{Hatzis, C.},
  \bibinfo{author}{Herman, D.}, \bibinfo{author}{Huang, J.},
  \bibinfo{author}{Jensen, R.V.}, \bibinfo{author}{Jiang, R.},
  \bibinfo{author}{Johnson, C.D.}, \bibinfo{author}{Jurman, G.},
  \bibinfo{author}{Kahlert, Y.}, \bibinfo{author}{Khuder, S.A.},
  \bibinfo{author}{Kohl, M.}, \bibinfo{author}{Li, J.}, \bibinfo{author}{Li,
  M.}, \bibinfo{author}{Li, Q.Z.}, \bibinfo{author}{Li, S.},
  \bibinfo{author}{Liu, J.}, \bibinfo{author}{Liu, Y.}, \bibinfo{author}{Liu,
  Z.}, \bibinfo{author}{Meng, L.}, \bibinfo{author}{Madera, M.},
  \bibinfo{author}{Martinez-Murillo, F.}, \bibinfo{author}{Medina, I.},
  \bibinfo{author}{Meehan, J.}, \bibinfo{author}{Miclaus, K.},
  \bibinfo{author}{Moffitt, R.A.}, \bibinfo{author}{Montaner, D.},
  \bibinfo{author}{Mukherjee, P.}, \bibinfo{author}{Mulligan, G.J.},
  \bibinfo{author}{Neville, P.}, \bibinfo{author}{Nikolskaya, T.},
  \bibinfo{author}{Ning, B.}, \bibinfo{author}{Page, G.P.},
  \bibinfo{author}{Parker, J.}, \bibinfo{author}{Parry, R.M.},
  \bibinfo{author}{Peng, X.}, \bibinfo{author}{Peterson, R.L.},
  \bibinfo{author}{Phan, J.H.}, \bibinfo{author}{Quanz, B.},
  \bibinfo{author}{Ren, Y.}, \bibinfo{author}{Riccadonna, S.},
  \bibinfo{author}{Roter, A.H.}, \bibinfo{author}{Samuelson, F.W.},
  \bibinfo{author}{Schumacher, M.M.}, \bibinfo{author}{Shambaugh, J.D.},
  \bibinfo{author}{Shi, Q.}, \bibinfo{author}{Shippy, R.}, \bibinfo{author}{Si,
  S.}, \bibinfo{author}{Smalter, A.}, \bibinfo{author}{Sotiriou, C.},
  \bibinfo{author}{Soukup, M.}, \bibinfo{author}{Staedtler, F.},
  \bibinfo{author}{Steiner, G.}, \bibinfo{author}{Stokes, T.H.},
  \bibinfo{author}{Sun, Q.}, \bibinfo{author}{Tan, P.Y.},
  \bibinfo{author}{Tang, R.}, \bibinfo{author}{Tezak, Z.},
  \bibinfo{author}{Thorn, B.}, \bibinfo{author}{Tsyganova, M.},
  \bibinfo{author}{Turpaz, Y.}, \bibinfo{author}{Vega, S.C.},
  \bibinfo{author}{Visintainer, R.}, \bibinfo{author}{von Frese, J.},
  \bibinfo{author}{Wang, C.}, \bibinfo{author}{Wang, E.},
  \bibinfo{author}{Wang, J.}, \bibinfo{author}{Wang, W.},
  \bibinfo{author}{Westermann, F.}, \bibinfo{author}{Willey, J.C.},
  \bibinfo{author}{Woods, M.}, \bibinfo{author}{Wu, S.}, \bibinfo{author}{Xiao,
  N.}, \bibinfo{author}{Xu, J.}, \bibinfo{author}{Xu, L.},
  \bibinfo{author}{Yang, L.}, \bibinfo{author}{Zeng, X.},
  \bibinfo{author}{Zhang, M.}, \bibinfo{author}{Zhao, C.},
  \bibinfo{author}{Puri, R.K.}, \bibinfo{author}{Scherf, U.},
  \bibinfo{author}{Tong, W.}, \bibinfo{author}{Wolfinger, R.D.},
  \bibinfo{year}{2010}.
\newblock \bibinfo{title}{The microarray quality control (maqc)-ii study of
  common practices for the development and validation of microarray-based
  predictive models}.
\newblock \bibinfo{journal}{Nat Biotechnol} \bibinfo{volume}{28},
  \bibinfo{pages}{827--838}.
\bibitem[{Yousef(2019)}]{Yousef2019AUCSmoothness-arxiv}
\bibinfo{author}{Yousef, W.A.}, \bibinfo{year}{2019}.
\newblock \bibinfo{title}{{AUC}: Nonparametric estimators and their
  smoothness}.
\newblock \bibinfo{journal}{arXiv preprint arXiv:1907.12851} .
\bibitem[{Yousef(2021)}]{Yousef2021EstimatingStandardErrorCross}
\bibinfo{author}{Yousef, W.A.}, \bibinfo{year}{2021}.
\newblock \bibinfo{title}{Estimating the standard error of
  cross-validation-based estimators of classifier performance}.
\newblock \bibinfo{journal}{Pattern Recognition Letters} \bibinfo{volume}{146},
  \bibinfo{pages}{115--145}.
\bibitem[{Yousef et~al.(2005)Yousef, Wagner and Loew}]{Yousef2005EstimatingThe}
\bibinfo{author}{Yousef, W.A.}, \bibinfo{author}{Wagner, R.F.},
  \bibinfo{author}{Loew, M.H.}, \bibinfo{year}{2005}.
\newblock \bibinfo{title}{{Estimating the Uncertainty in the Estimated Mean
  Area Under the ROC Curve of a Classifier}}.
\newblock \bibinfo{journal}{Pattern Recognition Letters} \bibinfo{volume}{26},
  \bibinfo{pages}{2600--2610}.
\bibitem[{Yousef et~al.(2006)Yousef, Wagner and Loew}]{Yousef2006AssessClass}
\bibinfo{author}{Yousef, W.A.}, \bibinfo{author}{Wagner, R.F.},
  \bibinfo{author}{Loew, M.H.}, \bibinfo{year}{2006}.
\newblock \bibinfo{title}{{Assessing Classifiers From Two Independent Data Sets
  Using ROC Analysis: a Nonparametric Approach}}.
\newblock \bibinfo{journal}{Pattern Analysis and Machine Intelligence, IEEE
  Transactions on} \bibinfo{volume}{28}, \bibinfo{pages}{1809--1817}.
\bibitem[{Zhang(1995)}]{Zhang1995AssessingPrediction}
\bibinfo{author}{Zhang, P.}, \bibinfo{year}{1995}.
\newblock \bibinfo{title}{{Assessing Prediction Error in Nonparametric
  Regression}}.
\newblock \bibinfo{journal}{Scandinavian Journal Of Statistics}
  \bibinfo{volume}{22}, \bibinfo{pages}{83--94}.

\end{thebibliography}
\end{document}